\documentclass{article}
\usepackage[preprint]{neurips_2020}
    \PassOptionsToPackage{numbers, compress}{natbib}

\newif\ifviz
\setcitestyle{square,numbers,sort}
\usepackage{algorithm, algorithmic}
\usepackage[utf8]{inputenc} 
\usepackage[T1]{fontenc}    
\usepackage{hyperref}       
\usepackage{url}            
\usepackage{booktabs}       
\usepackage{amsfonts}       
\usepackage{nicefrac}       
\usepackage{microtype}      
\usepackage{bm}
\usepackage{amsmath}
\usepackage{amsthm}
\usepackage{graphicx}
\usepackage{tikzinclude}
\usepackage{wrapfig}
\usepackage{tikz}
\usepackage{multirow}

\usepackage{pgfplots}
\usepackage{pgfplotstable}
\usepackage{xcolor}
\usepackage{color}
\usepackage{stmaryrd}
\usepackage{amssymb}
\allowdisplaybreaks 

\pgfplotsset{every axis/.append style={
                   label style={font=\footnotesize},
                   tick label style={font=\footnotesize},
		   ylabel style={rotate=-90},
		   yticklabel style={/pgf/number format/.cd,fixed relative,precision=1,/tikz/.cd},
                    }}
\pgfplotsset{compat=1.12}
\pgfplotsset{compat=1.12}
\newlength\figureheight
\newlength\figurewidth
\title{Online Matrix Completion with Side Information}
\usepackage{times}
 \author{Mark Herbster, Stephen Pasteris, Lisa Tse \\
Department of Computer Science\\ 
University College London\\ 
London WC1E 6BT, England, UK \\
 \texttt{\{m.herbster,s.pasteris,l.tse\}@cs.ucl.ac.uk}\\
}

\newtheorem{theorem}{Theorem} 
\newtheorem{proposition}[theorem]{Proposition}
\newtheorem{lemma}[theorem]{Lemma}

\newtheorem{definition}[theorem]{Definition}

\newcommand{\norm}[1]{\left\lVert#1\right\rVert}

\newcommand{\mc}{\operatorname{mc}}
\newcommand{\sign}{\operatorname{sign}}

\newcommand{\argmin}{\operatornamewithlimits{argmin}}
\newcommand{\into}{\rightarrow}
\newcommand{\bs}[1]{\boldsymbol{#1}}

\newcommand{\con}[2]{\left[#1;#2\right]}

\newcommand{\cl}[1]{\bs{q}^{#1}}

\newcommand{\imat}{\bm{Z}}
\bmdefine{\mat}{\tilde{U}}

\newcommand{\hl}{h}

\newcommand{\hlm}{h_{\mar}}

\newcommand{\la}{\bs{L}}

\newcommand{\re}{\mathcal{R}}

\newcommand{\one}{\bs{1}}

\newcommand{\wem}[1]{\bm{\tilde{W}}^{#1}}
\newcommand{\id}{\bs{I}}

\newcommand{\rent}[2]{\Delta(#1,#2)}

\newcommand{\ini}[1]{\left\|#1\right\|}

\newcommand{\di}[1]{\eta^t}
\newcommand{\pdi}[1]{\bar{\eta}^t}

\newcommand{\q}[1]{#1}

\newcommand{\Pm}{m}
\newcommand{\m}{\q{\Pm}}

\newcommand{\mar}{\gamma}
\newcommand{\marh}{\gamma}
\newcommand{\bx}[1]{\bs{x}^{#1}}
\newcommand{\bbx}[1]{\bs{X}^{#1}}

\newcommand{\co}{c}
\newcommand{\Exp}{\mathbb{E}}
\newcommand{\nm}{|\mathbb{M}|}
\newcommand{\nmset}{\mathbb{M}}
\newcommand{\uset}{\mathbb{U}}

\newcommand{\trace}[1]{\operatorname{tr}({#1})}
\newcommand{\trc}{\operatorname{tr}}
\newcommand{\tr}[1]{\operatorname{tr}\left(#1\right)}

\newcommand{\lr}{\eta}
\newcommand{\lrreg}{\eta}

\newcommand{\y}[1]{y_{#1}}

\newcommand{\dotp}[1]{{\langle{#1}\rangle}}
\newcommand{\vn}[1]{\norm{#1}}
\newcommand{\tnorm}[1]{\|{#1}\|_1}
\newcommand{\maxnorm}[1]{{\|{#1}\|_{\text{max}}}}

\newcommand{\gm}{\Re^{m\times n}}

\newcommand{\gmd}{\Re^{m \times d}}
\newcommand{\gnd}{\Re^{n \times d}}
\newcommand{\sm}{\{-1,1\}^{m\times n}}

\newcommand{\bikl}{\mathbb{B}_{k,\ell}^{m,n}}
\newcommand{\SP}{{\rm SP}}
\newcommand{\trans}{{\scriptscriptstyle \top}}

\newcommand{\bone}{{\bm 1}}
\newcommand{\bzero}{{\bm 0}}

\newcommand{\bw}{{\bm w}}

\newcommand{\ba}{\bm{a}}

\newcommand{\bb}{\bm{b}}
\newcommand{\bc}{\bm{c}}
\newcommand{\be}{\bm{e}}

\newcommand{\bp}{\bm{p}}
\newcommand{\bq}{\bm{q}}
\newcommand{\br}{\bm{r}}
\newcommand{\bt}{\bm{t}}
\newcommand{\bu}{\bm{u}}
\newcommand{\bmv}{\bm{v}}
\newcommand{\bxx}{\bm{x}}
\newcommand{\by}{\bm{y}}

\newcommand{\bxt}{\bm{\tilde{x}}}
\newcommand{\St}{\tilde{S}}

\newcommand{\bA}{\bm{A}}
\newcommand{\bB}{\bm{B}}
\newcommand{\bC}{\bm{C}}
\newcommand{\bD}{\bm{D}}

\newcommand{\bH}{\bm{H}}
\newcommand{\bI}{\bm{I}}
\newcommand{\bK}{\bm{K}}
\newcommand{\bLb}{\bm{L^{\circ}}}
\newcommand{\bL}{\bm{L}}
\newcommand{\bM}{\bm{M}}
\newcommand{\bN}{\bm{N}}
\newcommand{\bP}{\bm{P}}
\newcommand{\bQ}{\bm{Q}}
\newcommand{\bR}{\bm{R}}
\newcommand{\bS}{\bm{S}}
\newcommand{\bU}{\bm{U}}
\newcommand{\bV}{\bm{V}}
\newcommand{\bW}{\bm{W}}
\newcommand{\bWtilde}{\bm{\tilde{W}}}
\newcommand{\bX}{\bm{X}}
\newcommand{\bY}{\bm{Y}}
\newcommand{\bZ}{\bm{Z}}
\newcommand{\bLam}{\bm{\Lambda}}

\newcommand{\cE}{\mathcal{E}}

\newcommand{\cG}{\mathcal{G}}
\newcommand{\cH}{\mathcal{H}}
\newcommand{\cI}{\mathcal{I}}
\newcommand{\cJ}{\mathcal{J}}
\newcommand{\cK}{\mathcal{K}}
\newcommand{\cM}{\mathcal{M}}
\newcommand{\cN}{\mathcal{N}}
\newcommand{\cO}{\mathcal{O}}
\newcommand{\cOT}{\mathcal{\tilde{O}}}

\newcommand{\cU}{\mathcal{U}}
\newcommand{\cV}{\mathcal{V}}
\newcommand{\cX}{\mathcal{X}}
\newcommand{\onehalf}{\frac{1}{2}}

\newcommand{\PP}{\bP}
\newcommand{\QQ}{\bQ}
\newcommand{\nP}{\hat{\bP}} 
\newcommand{\nQ}{\hat{\bQ}} 
\newcommand{\R}{\bR} 
\newcommand{\C}{\bC} 
\newcommand{\RN}{\bM} 
\newcommand{\CN}{\bN} 
\newcommand{\RNfunc}{\mathcal{\bM^+}} 
\newcommand{\CNfunc}{\mathcal{\bN^+}} 
\newcommand{\SRN}{\sqrt{\bM}} 
\newcommand{\SCN}{\sqrt{\bN}} 
\newcommand{\RRN}{\mathcal{R}_{\RN}} 
\newcommand{\RCN}{\mathcal{R}_{\CN}} 
\newcommand{\IRRN}{\mathcal{R}_{\cM}} 
\newcommand{\IRCN}{\mathcal{R}_{\cN}} 
\newcommand{\RAD}{\mathcal{R}} 
\newcommand{\U}{\bU}
\newcommand{\Uset}{{\{-1,1\}^{m \times n}}}
\newcommand{\Up}{\bm{\bar{U}}}
\newcommand{\Ustar}{\U^*}

\newcommand{\nPstar}{\hat{\bP^*}}
\newcommand{\nQstar}{\hat{\bQ^*}}
\newcommand{\RNM}{\mathcal{N}}
\newcommand{\BEM}{\mathcal{B}}
\newcommand{\SPDM}{\bS_{++}}
\newcommand{\SPSDM}{\bS_{+}}

\newcommand{\eit}{\be_m^{i_t}}
\newcommand{\ejt}{\be_n^{j_t}}

\newcommand{\psm}{\sqrt{{\RN}^{+}}}
\newcommand{\ppsm}{\sqrt{\CN^{+}}}
\bmdefine{\mat}{\tilde{U}}
\newcommand{\XT}{\bm{\tilde{X}}^t}

\newcommand{\scp}{\mathcal{\widehat{\mathcal{D}}}}
\newcommand{\upD}{\widehat{\mathcal{D}}}
\newcommand{\qD}{\mathcal{D}^{\marh}_{\RN,\CN}}

\newcommand{\qDBF}{\mathcal{D}^{\circ}_{\RN,\CN}}

\newcommand{\pdim}{\mathcal{D}}
\newcommand{\pdimb}{\mathcal{{D^{\circ}}}}

\newcommand{\yut}{\Up_{i_t,j_t}}
\newcommand{\ybt}{\bar{y}_t}
\newcommand{\yht}{\hat{y}_t}
\newcommand{\Yrv}{Y_t}

\newcommand{\hazankale}{HKSS12}

\newcommand{\signcomplexity}{CMSM07}
\newcommand{\novikoff}{Novikoff62}
\newcommand{\matrixwinnow}{warmuth2007winnowing}

\newcommand{\MJH}[1]{{\color{purple}{{\rm\bfseries MJH:[}{\sffamily #1}{\rm\bfseries ]~}}}}
\newcommand{\LT}[1]{{\color{purple}{{\rm\bfseries LT:[}{\sffamily #1}{\rm\bfseries ]~}}}}
\newcommand{\HC}[1]{{}}

\begin{document}

\maketitle
\begin{abstract}
We give an online algorithm and prove novel mistake and regret bounds for online binary matrix completion with side information.  The mistake bounds we prove are of the form $\cOT(\frac{\pdim}{\mar^2})$.  The term $\frac{1}{\mar^2}$ is analogous to the usual margin term in SVM (perceptron) bounds.  More specifically, if we assume that there is some factorization of the underlying $m\times n$ matrix into $\bP \bQ^{\trans}$ where the rows of $\bP$ are interpreted as ``classifiers'' in $\Re^d$ and the rows of $\bQ$ as ``instances'' in $\Re^d$, then $\mar$ is is the maximum (normalized) margin over all factorizations $\bP \bQ^{\trans}$ consistent with the observed matrix.  The quasi-dimension term $\pdim$ measures the quality of side information.  In the presence of vacuous side information, $\pdim = m+n$.  However, if the side information is predictive of the underlying factorization of the matrix, then in an ideal case, $\pdim \in \cO(k + \ell)$ where $k$ is the number of distinct row factors and $\ell$ is the number of distinct column factors.   We additionally provide a generalization of our algorithm to the inductive setting.  
In this setting, we provide an example where the side information is not directly specified in advance.  For this example, the quasi-dimension $\pdim$ is now bounded  by $\cO(k^2 + \ell^2)$.
\end{abstract}
\section{Introduction}
We consider the problem of online binary matrix completion with {\em side information}.  In our model, the learner is sequentially queried to predict entries of a binary matrix. After each query, the learner then receives the value of that matrix entry.  The goal of the learner is to minimize prediction mistakes.  To aid the learner, side information is associated with each row and column. For instance, in the classic ``Netflix challenge''~\cite{BL07}, the rows of the matrix correspond to viewers and the columns to movies, with entries representing movie ratings.  It is natural to suppose that we have side information in the form of demographic information for each user,  and metadata for the movies.  In this work, we consider both {\em transductive} and {\em inductive} models.  In the former model, the side information associated with each row and column is specified completely in advance in the form of a pair of positive definite matrices that inform similarity between row pairs and column pairs.  For the inductive model, a pair of kernel functions is specified over potentially continuous domains, one for the rows and one for the columns. 
What is not specified is the mapping from the domain of the kernel function to specific rows or columns, which is only revealed sequentially.  In the Netflix example, the inductive model is especially natural if new users or movies are introduced during the learning process.

In Theorem~\ref{thm:base}, we will give  regret and mistake bounds for online binary matrix completion with side information.  Although this theorem has a broader applicability, our interpretation will focus on the case that the matrix has a {\em latent block structure}.  Hartigan~\cite{H72} introduced the idea of permuting a matrix by both the rows and columns into a few homogeneous blocks.  This is equivalent to assuming that each row (column) of the matrix has an associated row (column) class and that the matrix entry is completely determined by its corresponding row and column classes.  This has since become known as co- or bi-clustering.   This same assumption has become the basis for probabilistic models which can then be used to ``complete'' a matrix with missing entries.  The authors of~\cite{GLMZ16} give some rate-optimal results for this problem in the batch setting and provide an overview of this literature.  It is natural to compare this assumption to the dominant alternative, which assumes that there exists a low rank decomposition of the matrix to be completed, see for instance~\cite{CR12}.  Common to both approaches is that associated with each row and column, there is an underlying {\em latent} factor so that the given matrix entry is determined by a function on the appropriate row and column factor.  The low-rank assumption is that the latent factors are vectors in $\Re^d$  and that the function is the dot product. 
The latent block structure assumption is that the latent factors are instead categorical  and that the function between factors is arbitrary. 

In this work, we prove mistake bounds of the form $\cOT({\pdim}/{\mar^2})$.  The term $1/\mar^2$ is a parameter of our algorithm which, when exactly tuned, is the squared margin complexity $\mc(\bU)^2$ of the comparator matrix $\bU$.  The notion of margin complexity in machine learning was introduced in~\cite{BD2003}, where it was used to study the learnability of concept classes via linear embeddings. It was
further studied in~\cite{CMSM07}, and in~\cite{SrebroS05} a detailed study of margin complexity, trace complexity and rank in the context of statistical bounds for matrix completion was given.  The squared margin complexity is upper bounded by rank. Furthermore, if our $m \times n$ matrix has a latent block structure with $k \times \ell$ homogeneous blocks (for an illustration, see Figure~\ref{fig:exbcm}), then $\mc(\bU)^2 \le \min(k,\ell)$. 
The second term in our bound is the quasi-dimension $\pdim$ which, to the best of our knowledge, is novel to this work.  The quasi-dimension measures the extent to which the side information is ``predictive'' of the comparator matrix.  In Theorem~\ref{thm:bndpdim}, we provide an upper bound on the quasi-dimension, which measures the predictiveness of the side information when the comparator matrix has a latent block structure.  If there is only vacuous side information, then $\pdim = m+n$.  However, if there is a $k\times \ell$ latent block structure  and the side information is predictive, then $\pdim \in \cO(k+\ell)$; hence our nomenclature ``quasi-dimension.''   
In this case, we then have that the mistake bound term $\frac{\pdim}{\mar^2}\in\cO(k\ell)$, which we will later argue is optimal up to logarithmic factors.   Although latent block structure may appear to be a ``fragile'' measure of matrix complexity, our regret bound implies that performance will scale smoothly in the case of adversarial noise. 

The paper is organized as follows. First, we discuss related literature.  We then introduce preliminary concepts in Section~\ref{sec:preliminaries}. In Section~\ref{sec:basic}, we present our online matrix completion algorithm as well as a theorem (Theorem~\ref{thm:base}) that characterizes its performance in the transductive setting.  In Section~\ref{sec:lbs}, we formally introduce the concept of latent block structure (Definition~\ref{def:bikl}) and provide an upper bound (Theorem~\ref{thm:bndpdim}) for the quasi-dimension $\pdim$ when the matrix has latent block structure.  We then provide an example that bounds $\pdim$ when we have graph-based side information (Section~\ref{section:graph_side_info}); and a further example (Section~\ref{sec:sim}) when the matrix has additionally a ``community'' structure. Finally, in Section~\ref{sec:inductive}, we present an algorithm for the inductive setting, as well as an example illustrating a bound on $\pdim$ when the side information comes as vectors in $\Re^d$ which are separated by a clustering via hyper-rectangles.  Proofs as well as an experiment on synthetic data are contained in the appendices.
\subsection*{Related literature}
Matrix completion has been studied extensively in the batch setting, see for example~\cite{Srebro2005,CT10,Maurer2013,Chiang2018} and references therein. Central to these approaches is the aim of finding a low-rank factorization by optimizing a convex proxy to rank, such as the trace norm~\cite{Fazel2001}. The following papers~\cite{abernethy2006,Xu2013,Kalofolias2014,RHRD15} are partially representative of methods to incorporate side-information into the matrix completion task.  The inductive setting for matrix completion has been studied in~\cite{abernethy2006} through the use of tensor product kernels, and~\cite{Zhang2018} takes a non-convex optimization approach. Some examples in the transductive setting include~\cite{Xu2013,Kalofolias2014,RHRD15}. The last two papers use graph Laplacians to model the side information, which is similar to our approach. To achieve this, two graph Laplacians are used to define regularization functionals for both the rows and the columns so that rows (columns) with similar side information tend to have the same values. 
In particular, \cite{RHRD15} resembles our approach by applying the Laplacian regularization functionals to the underlying row and column factors directly. An alternate approach is taken in~\cite{Kalofolias2014}, where the regularization is instead applied to the row space (column space) of the ``surface'' matrix. 

In early work, the authors of~\cite{GRS93,GW95} proved mistake bounds for learning a binary relation which can be viewed as a special case of matrix completion.  In the regret setting, with minimal assumptions on the loss function,  the regret of the learner is bounded in terms of the {\em trace-norm} of the 
underlying comparator matrix in~\cite{CS11}. The authors of~\cite{\hazankale} provided tight upper and lower bounds in terms of a parameterized complexity class of matrices that include the bounded-trace-norm and bounded-max-norm matrices as special cases.  None of the above references considered the problem of side information.
The results in~\cite{gentile2013online,ourJMLR15,HPP16} are nearest in flavor to the results given here.  In~\cite{HPP16}, a mistake bound of $\cOT((m+n) \mc(\bU)^2)$ was given.   Latent block structure was also introduced to the online setting in~\cite{HPP16}; however, it was treated in a limited fashion and without the use of side information.  The papers~\cite{gentile2013online,ourJMLR15} both used side information to predict a limited complexity class of matrices.  In~\cite{gentile2013online}, side information was used to predict if vertices in a graph are ``similar''; in Section~\ref{sec:sim} we show how this result can be obtained as a special case of our more general bound.  In~\cite{ourJMLR15}, a more general setting was considered, which as a special case addressed the problem of a switching graph labeling.
The model in~\cite{ourJMLR15} is considerably more limited in its scope than our Theorem~\ref{thm:base}.  To obtain our technical results, we used an adaptation of the matrix exponentiated gradient algorithm~\cite{tsuda2005matrix}.  The general form of our regret bound comes from a matricization of the regret bound proven for a Winnow-inspired algorithm~\cite{litt88} for linear classification in the vector case given in~\cite{Sabato2015}. For a more detailed discussion, see Appendix~\ref{ap:regretmw}.

\vspace{-0.1in}
\section{Preliminaries} \label{sec:preliminaries}
For any positive integer $m$, we define $[m] := \left\{1,2,\ldots,m\right\}$. For any predicate $[\mbox{\sc pred}] :=1$ if $\mbox{\sc pred}$ is true and equals 0 otherwise, and $[x]_{+} := x [x>0]$.  

We denote the inner product of vectors $\bxx,\bw\in\Re^n$ as $\dotp{\bxx,\bw} = \sum_{i=1}^n x_i w_i$ and the norm as $\vn{\bw} = \sqrt{\dotp{\bw,\bw}}$.  The $ith$ coordinate $m$-dimensional vector  is denoted $\be_m^i := ([j=i])_{j\in [m]}$; we will often abuse notation and use $\be^i$ on the assumption that the dimensionality of the space may be inferred.  For vectors $\bp \in \mathbb{R}^m$ and $\bq \in \mathbb{R}^n$ we define $\con{\bp}{\bq}\in \mathbb{R}^{m+n}$ to be the concatenation of $\bp$ and $\bq$, which we regard as a column vector. Hence $\con{\bp}{\bq}^\trans\!{\con{\boldsymbol{\bar \bp}}{\boldsymbol{{\bar \bq}}}}=\bp^\trans \boldsymbol{\bar \bp}+\bq^\trans \boldsymbol{\bar \bq}$.  We let $\gm$ be the set of all $m\times n$ real-valued matrices.  If $\bX\in\gm$ then $\bX_i$ denotes the $i$-th $n$-dimensional row vector and the $(i,j)^{th}$ entry of $\bX$ is $X_{ij}$. We define $\bX^+$ and $\bX^\trans$ to be its pseudoinverse and transpose, respectively. 
 The trace norm of a matrix $\bX\in\gm$ is $\tnorm{\bX} = \trace{\sqrt{\bX^\trans \bX}}$, where $\sqrt{\cdot}$ indicates the unique positive square root of a positive semi-definite matrix, and $\trace{\cdot}$ denotes the trace of a square matrix. This is given by $\trace{\bY} = \sum_{i=1}^n Y_{ii}$ for $\bY\in \Re^{n\times n}$. 
 The $m \times m$ identity matrix is denoted $\bI^m$.   In addition, we define $\bS^m$ to be the set of $m \times m$ symmetric matrices and let $\bS^m_+$ and $\bS_{++}^m$ be the subset of positive semidefinite and strictly positive definite matrices respectively. Recall that the set of symmetric matrices $\bS_{+}^m$ has the following partial ordering: for every $\bM,\bN\in \bS_+^m$, we say that $\bM\preceq\bN$ if and only if $\bN-\bM \in \bS_+^m$. 
We also define the squared radius of $\RN \in \bS_+^m$ as $\RRN := \max_{i\in [m]} M^+_{ii}$.

For every matrix $\bU \in \Re^{m \times n}$, we define $\SP(\bU) = \{\bV \in \gm :\forall_{ij} V_{ij} U_{ij} > 0\}$, the set of matrices which are sign consistent with $\bU$. We also define $\SP^1(\bU) = \{\bV \in \gm : \forall_{ij} V_{ij} \sign(U_{ij}) \geq 1\}$, that is the set of matrices which are sign consistent  with $\bU$ with a margin of at least one.

The max-norm (or $\gamma_2$ norm \cite{CMSM07}) of a matrix $\bU \in \gm$ is defined by 
\begin{equation}
\label{eq:maxnorm}
\maxnorm{\U} := \min_{\PP \QQ^\trans = \U} \left\{\max_{1\leq i\leq m} \vn{\PP_{i}} ~\max_{1\leq j\leq n}\vn{\QQ_{j}}\right\}\,,
\end{equation} 
where the minimum is over all matrices $\PP \in \gmd$ and $\QQ \in \gnd$ and every integer $d$.
The {\em margin complexity} of a matrix  $\U\in\gm$ is 
\begin{equation}\label{eq:mcdef}
\mc(\U) := \min_{\bV\in \SP^1(\U)}  \maxnorm{\bV} = \!\!\!\min_{\PP \QQ^{\trans}\in\SP(\U)} \max_{ij}
\frac{\norm{\bP_{i}}{\norm{\bQ_{j}}}}{|\dotp{\bP_{i},\bQ_{j}}|}\,.
\end{equation}
Observe that for $\U\in\{-1,1\}^{m\times n}$, $1 \le \mc(\U) \le \maxnorm{\U}\le \min(\sqrt{m},\sqrt{n})$, where the lower bound follows from the right hand side of~\eqref{eq:mcdef} and the upper bound follows since we may decompose $\U = \U \bI^n$ or as $\U = \bI^m \U$.  Note there may be a large gap between the margin complexity and the max-norm.  In~\cite{\signcomplexity} a matrix in $\U\in \{-1,1\}^{n\times n}$ was given such that $\mc(\U) = \log n$ and $\maxnorm{\U} = \Theta\left(\frac{\sqrt{n}}{\log n}\right)$.
We denote the classes of $m \times d$ {\em row-normalized} and {\em block expansion} matrices as
$\RNM^{m,d} := \{\hat{\bP} \subset \Re^{m \times d} : \norm{\hat{\bP}_{i}} =1, i\in [m]\}$ and 
$\BEM^{m,d} := \{\bR \subset \{0,1\}^{m \times d} : \norm{\bR_i} =1, i\in [m],\operatorname{rank}(\bR) = d\}$, respectively.  Block expansion matrices may be seen as a generalization of permutation matrices, additionally duplicating rows (columns) by left (right) multiplication.  We define the {\em quasi-dimension} of a matrix  $\U\in\gm$ with respect to $\RN\in\bS_{++}^m,\, \CN\in\bS_{++}^n$ at margin~$\marh$ as
\begin{equation}\label{eq:defpdim}
\qD(\bU) := \min_{\nP \nQ^\trans = {\marh}\U} \RRN\trc\left(\nP^{\trans}\RN\nP  \right) 
+
\RCN\trc\left(\nQ^{\trans}\CN\nQ \right)\,,
\end{equation}
where the infimum is over all row-normalized matrices $\nP \in \RNM^{m,d}$ and $\nQ \in \RNM^{n,d}$ and every integer~$d$. If the infimum does not exist then $\qD(\bU) :=+\infty$.  Note that the infimum exists iff
$\maxnorm{\U} \le 1/\marh$.   Finally note that $\qD(\bU) = m + n$ if $\maxnorm{\U} \le 1/\marh$,
$\RN = \bI^m$ and $\CN = \bI^n$.

We now introduce notation specific to the graph setting.
Let then $\cG = (\cV,\cE,\bW)$ be an $m$-vertex  connected, weighted and undirected graph with positive weights.  
Let $\bA$ be the  $m\times m$ matrix such that $A_{ij}:= A_{ji}=W_{ij}$ if $(i,j)\in \cE(\cG)$ and $A_{ij}:= 0$ otherwise. Let $\bD$ be the $m\times m$ diagonal matrix such that $D_{ii}$ is the degree of vertex $i$. The Laplacian, $\bL$, of $\cG$ is defined as $\bD-\bA$.  Observe that if $\cG$ is connected, then $\bL$ is rank $m-1$ matrix with $\bone$ in its null space.  From $\bL$ we define the (strictly) positive definite {\em  PDLaplacian} 
$\bLb := \la+\left(\frac{\one}{\m}\right)\left(\frac{\one}{\m}\right)^\trans\RAD_{\bL}^{-1}$.  Observe that if $\bu\in [-1,1]^m$ then 
$(\bu^{\trans} \bLb \bu)  \RAD_{\bL^{\circ}}\le 2 (\bu^{\trans} \bL \bu\, \RAD_{\bL} +  1) $,
and similarly, 
$(\bu^{\trans} \bL \bu) \RAD_{\bL} \le \onehalf (\bu^{\trans} \bLb \bu)  \RAD_{\bL^{\circ}}$ 
(see~\cite{herbster2006prediction} for details of this construction).
\section{Transductive Matrix Completion}\label{sec:basic}
Algorithm~\ref{alg:base} corresponds to an adapted {\sc Matrix Exponentiated Gradient} ({\sc MEG}) algorithm~\cite{tsuda2005matrix} to perform transductive matrix completion with side information.   Although the algorithm is a special case of {\sc MEG}, the following theorem does not follow as a special case of the analysis in~\cite{tsuda2005matrix}. 
\begin{algorithm}[h]
\footnotesize
\begin{algorithmic} 
\caption{Predicting a binary matrix with side information in the transductive setting. \label{alg:base}}
\renewcommand{\algorithmicrequire}{\textbf{Parameters:}} 
\REQUIRE Learning rate: $0<\lr$\ , quasi-dimension estimate: $1\le \upD$, margin estimate: $0 <\marh\le 1$, non-conservative flag $[\mbox{\sc non-conservative}]\in \{0,1\}$
and side information matrices $\RN\in\bS_{++}^m,\, \CN\in\bS_{++}^n$ with $m+n \geq 3$\\  \vspace{.1truecm}
\renewcommand{\algorithmicrequire}{\textbf{Initialization:}}
\REQUIRE $\nmset \leftarrow \emptyset\ ; \ \wem{1}\leftarrow\frac{\upD}{(m+n)} \id^{m+n}$. \vspace{.1truecm}
\renewcommand{\algorithmicrequire}{\textbf{For}}
\REQUIRE $t =1,\dots,T$  \vspace{.1truecm}
\STATE $\bullet$ Receive pair $(i_t,j_t) \in [m] \times [n].$ \\ \vspace{.1truecm}
\STATE $\bullet$ Define 
\begin{equation}\label{eq:bbxdef}
\XT:=  \bx{t}(\bx{t})^\trans := \con{\frac{\psm\eit}{\sqrt{2\RRN}}}{\frac{\ppsm\ejt}{\sqrt{2\RCN}}}  \con{\frac{\psm\eit}{\sqrt{2\RRN}}}{\frac{\ppsm\ejt}{\sqrt{2\RCN}}}^\trans\,.
\end{equation}\vspace{-.2in}
\STATE $\bullet$  Predict 
\begin{equation*} \Yrv \sim \mbox{\sc Uniform}(-\marh,\marh) \!\times\! [\mbox{\sc non-conservative}]\,;\  
\ybt \!\leftarrow\! \tr{\wem{t}\XT}-1 \,;\  \yht\! \leftarrow\! \sign(\ybt - \Yrv)\,.\vspace{-.2in}
\end{equation*}
\STATE $\bullet$ Receive label $\y{t} \in \{-1,1\}$\,.\vspace{.1truecm}
\STATE $\bullet$ If $y_t \ne \yht$ then $\nmset \leftarrow \nmset \cup \{t\}.$
\STATE $\bullet$ If $y_t\ybt < \marh \times [\mbox{\sc non-conservative}]$ then \vspace{-.08in}
\begin{equation*}
\wem{t+1} \leftarrow\exp\left(\log(\wem{t})+ \lr\y{t}\XT\right)\,.
\end{equation*}\vspace{-.15in}
\STATE $\bullet$ Else  $\wem{t+1}  \leftarrow \wem{t}$.
\end{algorithmic}
\end{algorithm}
In the following theorem we give an expected regret bound.  In the realizable case (with exact tuning), the mistakes are bounded by $\cOT(\pdim \mc(\bU)^2 )$.  The term $\pdim$ evaluates the predictive quality of the side information provided to the algorithm.  In order to evaluate $\pdim$, we provide an upper bound in Theorem~\ref{thm:bndpdim} that is more straightforward to interpret.  Examples  are given in Sections~\ref{section:graph_side_info} and~\ref{sec:boxstory}, where Theorem~\ref{thm:bndpdim} is applied to evaluate the quality of side information in idealized scenarios.
 \begin{theorem}\label{thm:base}
The expected regret of Algorithm~\ref{alg:base} with {\bf non-conservative} updates  ($[\mbox{\sc non-conservative}]=1$)  and parameters $\mar \in (0,1]$, $\upD \geq \qD(\bU)$ ,
$\lr = \sqrt{\frac{\upD \log(m+n) }{2 T}}$, p.d. matrices $\RN\in \SPDM^m$ and $\CN\in \SPDM^n$ is bounded  by
\begin{equation}\label{eq:baseregret}
\Exp[\nm] - \sum_{t\in [T]} [ y_t \ne U_{i_t j_t}] \le  4 \sqrt{2 \frac{\upD}{\mar^2}\log(m+n) T} \end{equation}
for all $\bU\in \sm$ with $\maxnorm{\bU} \le 1/\mar$.

The mistakes in the {\bf realizable} case with  {\bf conservative} updates ($[\mbox{\sc non-conservative}]=0$) and parameters
$1/\lr = 1/\mar\geq \mc(\bU)$, $\upD \geq \min\limits_{\bV \in \SP^1(\bU)}  \qD(\bV)$ and for  $T\ge 1$ are bounded by, 
\begin{equation}\label{eq:basereal}
\nm \le 3.6 \frac{\upD}{\mar^2}\log(m+n)\,,
\end{equation}
for all $\bU\in \sm$ with $\mc(\bU) \le 1/\mar$ and $y_t =  U_{i_t j_t}$ for all $t\in\nmset$. 
\end{theorem}
If the side information is vacuous, that is $\RN = \bI^m$ and $\CN = \bI^n$, then $\pdim = m+n$.   In this scenario, we recover a special 
case\footnote{In~\cite{\hazankale}, a regret bound for general loss functions for matrix completion without side information is given for $(\beta,\tau)$-decomposable matrices.  When $\beta$ is at its minimum over all possible decompositions, we recover the bound  up to constant factors with respect to the expected 0-1 loss.  On the algorithmic level, our works are similar except that the algorithm of~\cite{\hazankale} contains an additional projection step that dominates the computation time of the update.} of the analysis 
of~\cite{\hazankale} up to constant factors and with the additional assumption of realizability we recover~\cite[Theorem 3.1]{HPP16}.
The term $\pdim$ is difficult to directly quantify. 
In the next section, we specialize our analysis to the case that the matrix $\U$ has a latent block structure.  
\vspace{-0.1in}
\section{Latent Block Structure}\label{sec:lbs}
We introduce the concept class of $(k,\ell)$-binary-biclustered matrices 
(previously defined in~\cite[Section 5]{HPP16}), 
in the following definition.  We then give an upper bound to $\qD(\U)$ when a matrix has this type of latent structure in Theorem~\ref{thm:bndpdim}.  The magnitude of the bound will depend on how ``predictive'' matrices $\RN$ and $\CN$ are of the latent block structure.  In Sections~\ref{section:graph_side_info} and~\ref{sec:sim}, we will use a variant of the discrete Laplacian matrix for $\RN$ and $\CN$ to encode side information and illustrate the resultant bounds for idealized scenarios.
\begin{definition}
\label{def:bikl}
The class of $(k,\ell)$-binary-biclustered matrices is defined as
\begin{equation*}
\bikl = \{\bU \in \Uset : \br \in [k]^m, \bc \in [\ell]^n, \bU^* \in \{-1,1\}^{k\times\ell},~U_{ij} = U^*_{r_ic_j},~i \in [m], j \in [n] \} \,.
\end{equation*} 
\end{definition}
Thus each row $r_i$ is associated with a latent factor in $[k]$ and each column $c_j$ is associated with a latent factor in $[\ell]$ and the interaction of factors is determined by a matrix $\bU^* \in \{-1,1\}^{k\times\ell}$.
More visually, a binary matrix is $(k,\ell)$-biclustered if there exists some permutation of the rows and columns into a $k \times \ell$ grid of blocks each uniformly labeled $-1$ or $+1$, as illustrated in Figure~\ref{fig:exbcm}.   Determining if a matrix is in $\bikl$, may be done directly by a greedy algorithm.  However, the problem of determining if a matrix with missing entries may be completed to a matrix in $\mathbb{B}_{k,n}^{m,n}$ was shown in~\cite[Lemma 8]{GKOS18} to be {\sc NP-complete} by reducing the problem to  {\sc Clique Cover}.
\begin{figure}
\begin{center}
\includegraphics[width=.45\linewidth]{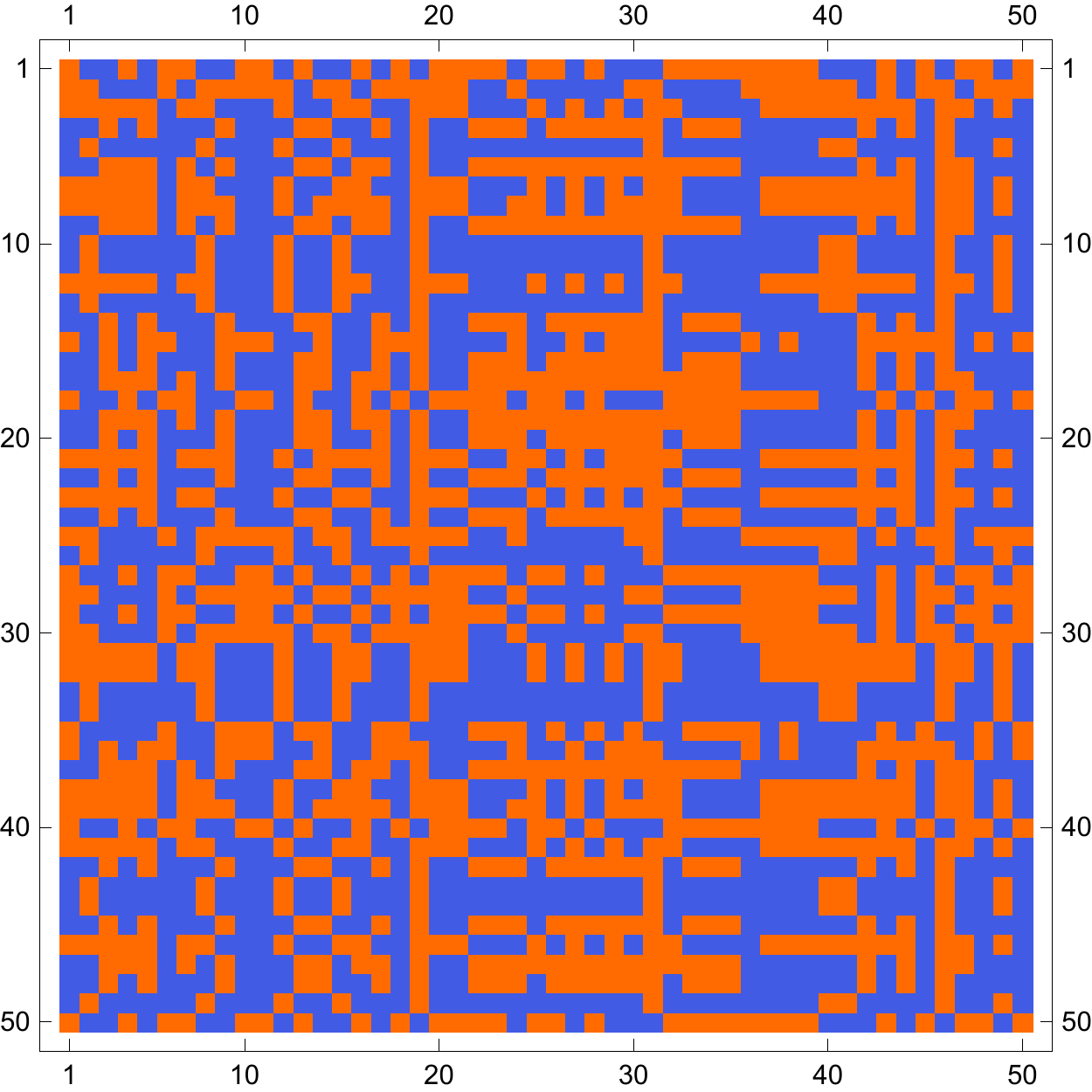}\ \ \ \includegraphics[width=.45\linewidth]{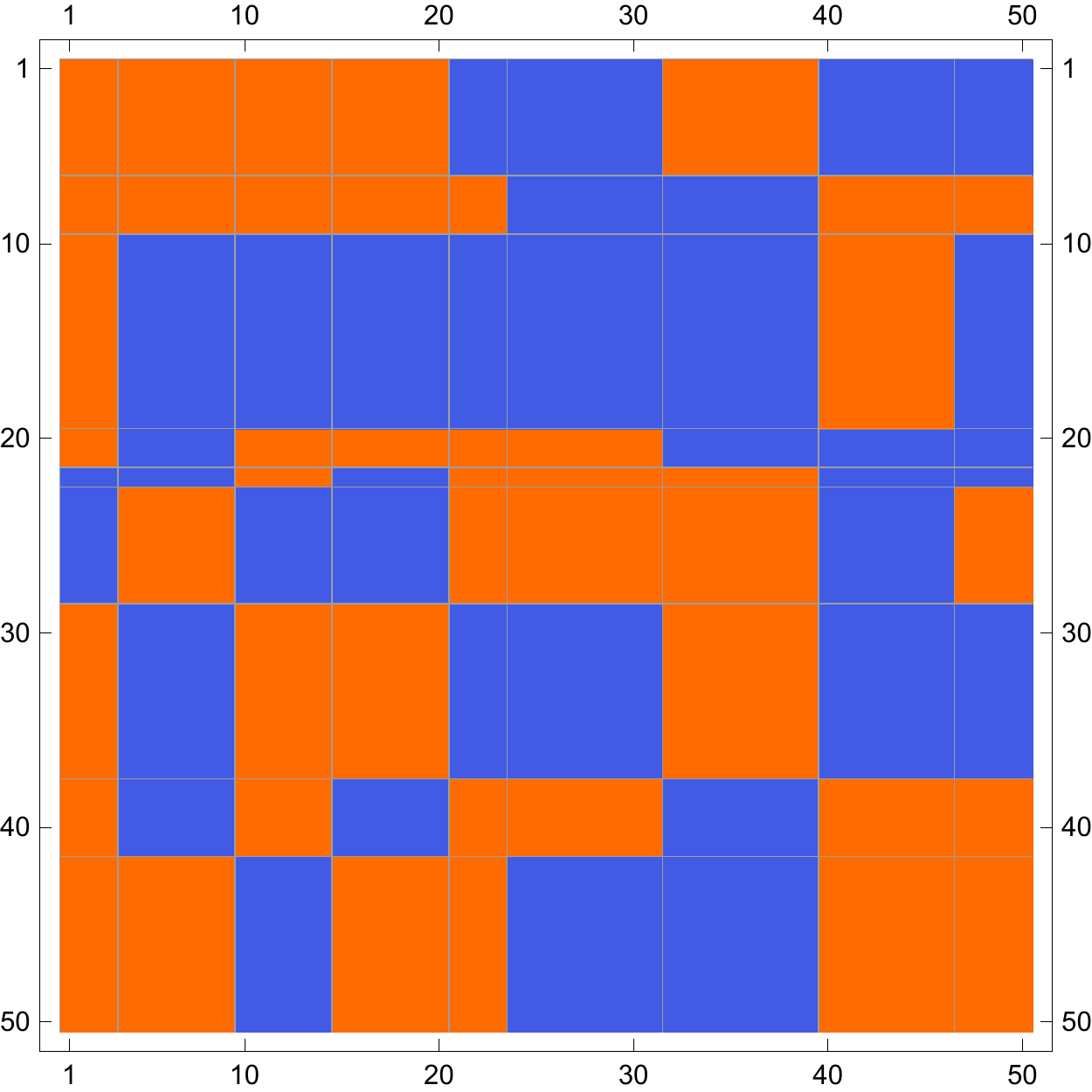}
\end{center}
\caption{A $(9,9)$-biclustered $50\times 50$ binary matrix before/after permuting into latent blocks.}\label{fig:exbcm}
\end{figure}

Many natural functions of matrix complexity are invariant to the presence of block structure.  A function $f:\cX \rightarrow \Re$ with respect to a class of matrices $\cX$  is {\em block-invariant} if for all $m,k,n,\ell\in\mathbb{N}^+$ with $m\ge k$, $n\ge\ell$, $\bR\in\BEM^{m,k}$ and $\bC\in\BEM^{n,\ell}$  we have that $f(\bX) = f(\bR\bX\bC^{\trans})$ for any $k \times \ell$ matrix $\bX\in\cX$.
The max-norm, margin complexity, rank and VC-dimension\footnote{Here, a hypothesis class $\cH$ defines a matrix via $\U := (h(x))_{h\in\cH,x\in\cX}$.} are all block-invariant. 
 From the block-invariance of the max-norm, we may conclude that for $\U\in\bikl$, 
 \begin{equation}\label{eq:mclbl}
 \mc(\U)\le\maxnorm{\U}=\maxnorm{\Ustar} \le \min(\sqrt{k},\sqrt{\ell}).
 \end{equation}
This follows since we may decompose $\bU= \bR \Ustar \bC^{\trans}$ for some $\Ustar\in\{-1,1\}^{k\times\ell}$, $\bR\in\BEM^{m,k}$ and $\bC\in\BEM^{n,\ell}$ and then use the  observation in the preliminaries that the max-norm of any matrix in $\Uset$ is bounded by $\min(\sqrt{m},\sqrt{n})$.

In the following theorem, we give a bound for the quasi-dimension $\qD(\bU)$ which will scale with the dimensions of the latent block structure and the ``predictivity'' of $\RN$ and $\CN$  with respect to that block structure.  The bound is independent of $\marh$ in so far as $\qD(\bU)$ is finite.

\begin{theorem}\label{thm:bndpdim}
If  $\U\in\bikl$ define 
\begin{equation}\label{eq:defpdimb}
\qDBF(\U) := 
\begin{cases}
 2 \trc( \bR^{\trans} \RN \bR)\RRN 
\!+ \!2 \trc( \bC^{\trans} \CN \bC)\RCN\!+\! 2k\! +\! 2\ell \, & \RN \text{ and }\CN \text{ are PDLaplacians} \\
k\trace{\bR^{\trans} \bM \bR}\RRN+\ell\trace{\bC^{\trans} \bN \bC}\RCN
 & \RN\in\SPDM^m \text{ and }\CN\in\SPDM^n
\end{cases}\,,
\end{equation}
as the minimum over all decompositions of $\U = \bR \Ustar \bC^{\trans}$ for
$\bR\in \BEM^{m,k}$, $\bC\in \BEM^{n,\ell}$ and $\Ustar \in \{-1,1\}^{k\times \ell}$.  
Thus for $\U\in\bikl$, 
\begin{align*}
\qD(\U) & \le \qDBF(\U)\quad\quad (\mbox{if }\maxnorm{\U} \le {1}/{\marh}) \\
\min_{\bV\in\SP^1(\U)} \qD(\bV) & \le \qDBF(\U) \quad\quad (\mbox{if }\mc(\U) \le {1}/{\marh})\,.
\end{align*}
\end{theorem}
The bound $\qD(\U) \le \qDBF(\U)$ 
allows us to bound the quality of the side information in terms of a hypothetical learning problem.  Recall that $\argmin_{r_i y_i \ge 1 : i\in [m]}  (\br^{\trans} \RN \br)  \RRN$ is the upper bound on the mistakes per Novikoff's theorem~\cite{\novikoff} for predicting the elements of vector $\by \in \{-1,1\}^m$ with a kernel perceptron using $\RN^{-1}$ as the kernel.  Hence the term $\cO(\trc( \bR^{\trans} \RN \bR)\RRN)$ in~\eqref{eq:defpdimb} may be interpreted as a bound for a one-versus-all $k$-class kernel perceptron where $\bR$ encodes a labeling from $[k]^m$ as one-hot vectors.  
We next show an example where $ \qDBF(\U) \in \cO(k+\ell)$ with ``ideal'' side information.
\vspace{-0.1in}
\subsection{Graph-based Side Information}\label{section:graph_side_info}
We may use a pair of separate graph Laplacians to represent the side information on the ``rows'' and the ``columns.''  A given row (column) corresponds to a vertex in the ``row graph'' (``column graph'').  The weight of edge $(i,j)$ represents our prior belief that row (column) $i$ and row (column) $j$ share the same underlying factor.   Such graphs may be inherent to the data. For example, we have a social network of users and a network based on shared actors or genres for the movies in a ``Netflix'' type scenario.  Alternatively, as is common in graph-based semi-supervised learning~\cite{BN04,Zhu2009} we may build a graph based on vectorial data  associated with the rows (columns), for example, user demographics.  Although the value of $\pdim$ will vary smoothly with the predictivity of $\RN$ and $\CN$ of the factor structure, in the following we give an example to quantify $\pdimb$ in a best case scenario.
\vspace{-0.1in}
\subsection*{Bounding $\pdimb$ for ``ideal'' graph-based side information}
In this ideal case we are assuming that we know the partition of $[m]$ that maps rows to factors.   The rows that share factors have an edge between them and there are no other edges. Given $k$ factors, we then have a graph that consists of $k$ disjoint cliques.  However, to meet the technical requirement that the side information matrix $\RN (\CN)$ is positive definite, we need to connect the cliques in a minimal fashion.   We achieve this by connecting the cliques like a ``star'' graph.  Specifically, a clique is arbitrarily chosen as the center and a vertex in that clique is arbitrarily chosen as the central vertex. From each of the other cliques, a vertex is chosen arbitrarily and connected to the central vertex.  Observe that a property of this construction is that there is a path of length $\le4$ between any pair of vertices. 
Now we can use the bound from Theorem~\ref{thm:bndpdim},
\[
\pdimb =
 2 \trc( \bR^{\trans} \RN \bR)\RRN 
+  2 \trc( \bC^{\trans} \CN \bC)\RCN+ 2k + 2\ell\,,
\]
to bound $\pdim \le \pdimb$ in this idealized case.  We focus on the rows, as a parallel argument may be made for the side information on the columns.  
Consider the term $\trc( \bR^{\trans} \RN \bR)\RRN$, where $\RN := \bLb$ is the PDLaplacian formed from a graph with Laplacian $\bL$. Then using the observation from the preliminaries that $(\bu^{\trans} \bLb \bu)  \RAD_{\bLb}\le 2 (\bu^{\trans} \bL \bu\, \RAD_{\bL} +  1)$, we have that $\trc( \bR^{\trans} \RN \bR)\RRN \le 2\trc(\bR^{\trans} \bL \bR)\RAD_{\bL} +2k$.  To evaluate this, we use the well-known equality of $\trc(\bR^{\trans} \bL \bR) = \sum_{(i,j)\in E} \norm{\bR_i - \bR_j}^2$. Observing that each of the $m$ rows of $\bR$ is a ``one-hot'' encoding of  the corresponding factor, only the edges between classes then contribute to the sum of the norms, and thus by construction  $\trc(\bR^{\trans} \bL \bR) \le k-1$.  We bound $\mathcal{R}_{\bL} \le 4$, using the fact that the graph diameter is a bound on $\mathcal{R}_{\bL}$ 
(See~\cite[Theorem 4.2]{herbster2005online}).   Combining terms and assuming similar idealized side information on the columns, we obtain $\pdimb \in O(k + \ell)$.  Observe then that since the comparator matrix is $(k,\ell)$-biclustered, we have in the realizable case (with exact tuning), that $\mc(\U)^2 \le \min(k,\ell)$ by~\eqref{eq:mclbl}. Thus, the mistakes of the algorithm are bounded by $\cOT(\mc(\U)^2\pdimb)=\cOT(k\ell)$. 
This upper bound is tight up to logarithmic factors as we may decompose $\bU= \bR \Ustar \bC^{\trans}$ for some $\Ustar\in\{-1,1\}^{k\times\ell}$, $\bR\in\BEM^{m,k}$ and $\bC\in\BEM^{n,\ell}$ and  force a mistake for each of the $k\ell$ entries in $\Ustar$.

Can side information provably help?  Unsurprisingly, yes. Consider the set of matrices such that each row  is either all `+1' or all `-1'.  This set is exactly $\mathbb{B}^{m,n}_{2,1}$.  Clearly an adversary can force $m$ mistakes whereas with ``ideal'' side information the upper bound is $\cOT(1)$.

Similar results to the above can be obtained via alternate positive definite embeddings. For example, consider a {\em $k$-partition} kernel of $[m]$ where $K_{\epsilon,S_1,\ldots,S_k}(i,j) := [i,j\in S_r : r\in[k]] +\epsilon [i=j]$ for some partition of $[m]$ into disjoint sets $S_1,\ldots,S_k$.  By using $\RN^{-1} = (K(i,j))_{i,j\in [m]}$ one can obtain for small~$\epsilon$,  bounds that are tighter than achieved by the Laplacian with respect to constant factors.  We have focused on the Laplacian as a method for encoding side information as it is more straightforward to  encode~\cite{Zhu2009} ``softer'' knowledge of relationships.
\vspace{-0.1in}
\subsection{Online Community Membership Prediction}\label{sec:sim}
A special case of matrix completion is the case where there are $m$ objects which are assumed to lie in $k$ classes (communities).   In this case, the underlying matrix $\bU\in\{-1,1\}$  is given by $U_{ij}=1$ if $i$ and $j$ are in the same class and $U_{ij}= -1$ otherwise.  Thus this may be viewed as an online version of community detection or ``similarity'' prediction.   In~\cite{gentile2013online}, this problem was addressed when the side information was encoded in a graph and the aim was to perform well when there were few edges between classes (communities).

Observe that this is an example of  a $(k,k)$-biclustered $m \times  m$ matrix where $\Ustar = 2 \bI^k - \bone \bone^{\trans}$ and there exists $\bR\in\BEM^{m,k}$ such that $\bU := \bR \Ustar\bR^{\trans}$.  Since the max-norm is block-invariant, we have that $\maxnorm{\bU} = \maxnorm{\bU^*}$. In the case of a general $k\times k$ biclustered matrix, $\maxnorm{\U^*}\le \sqrt{k}$ (see~\eqref{eq:mclbl}).  However in the case of ``similarity prediction'', we have $\maxnorm{\Ustar} \in O(1)$.  This follows since we have a decomposition $\Ustar = \bP\bQ^{\trans}$ by $\bP,\bQ\in\Re^{k,k+1}$ with $\bP := (P_{ij} = \sqrt{2} [i=j] + [j=k+1])_{i\in [k],j\in [k+1]}$ and $\bQ := (Q_{ij} = \sqrt{2} [i=j] - [j=k+1])_{i\in [k],j\in [k+1]}$,  thus giving $\maxnorm{\Ustar}\le 3$.  This example also shows that there may be an arbitrary  gap between rank and max-norm of $\pm 1$ matrices as the rank of $\Ustar$ is $k$  (in~\cite{CMSM07} this gap between the max-norm and rank was previously observed).  Therefore, if the side-information matrices are taken to be the same PDLaplacian $\RN=\CN$ defined from a Laplacian $\bL$,  we have that since $\maxnorm{\U}\in\cO(1)$ and $\pdimb\in\cO(\trc(\bR^{\trans} \bL\bR)\RAD_{\bL})$,  a mistake bound of $\cOT(\trc(\bR^{\trans} \bL \bR)\RAD_{\bL})$ is obtained, which recovers the bound of~\cite[Proposition 4]{gentile2013online} up to constant factors.  
This work extends the results in~\cite{gentile2013online} for similarity prediction to regret bounds, and to the inductive setting with general p.d. matrices.  In the next section, we will see how this type of result may be extended to an inductive setting.
\vspace{-0.1in}
\section{Inductive Matrix Completion}\label{sec:inductive}
In the previous section, the learner was assumed to have complete foreknowledge of the side information through the matrices $\RN$ and $\CN$.   In the inductive setting, the learner  has  instead kernel side information functions $\RNfunc$ and $\CNfunc$.  With complete foreknowledge of the rows (columns) that will be observed, one may use $\RNfunc$ ($\CNfunc$) to compute 
$\RN$ ($\CN$) which corresponds to an inverse of a submatrix of $\RNfunc$ ($\CNfunc$).  In the inductive, unlike the transductive setting, we do not have this foreknowledge and thus cannot compute $\RN$ ($\CN$) in advance.
 Notice that the assumption of side information as kernel functions is not particularly limiting, as for instance the side information could be provided by vectors in $\Re^d$ and the kernel could be the positive definite linear kernel 
$\bK_{\epsilon}(\bxx,\bxx') :=  \dotp{\bxx,\bxx'} + \epsilon [\bxx = \bxx']$.  On the other hand, despite the additional flexibility of the inductive setting versus the transductive one, there are two limitations.  First, only in a technical sense will it be possible to model side information via a PDLaplacian, since $\RN^+$ can only be computed given knowledge of the graph in advance. 
Second, the bound in Theorem~\ref{thm:bndpdim} on the quasi-dimension $\pdim \le \pdimb$ gains  additional multiplicative factors $k$ and $\ell$.  Nevertheless, we will observe in Section~\ref{sec:boxstory} that, for a given kernel for which the side information associated with a given row (column) latent factor is ``well-separated'' from distinct latent factors, we can show that $\pdimb \in \cO(k^2 + \ell^2)$.

The following algorithm is prediction-equivalent to Algorithm~\ref{alg:base} up to the value of $\RRN (\RCN)$.  In~\cite{WKZ12}, the authors provide very general conditions for the ``kernelization'' of algorithms with an emphasis on ``matrix'' algorithms.  They sketch a method to kernelize the {\sc Matrix Exponentiated Gradient} algorithm based on the relationship between the eigensystems of the kernel matrix and the Gram matrix.  We take a different, more direct approach in which we prove its correctness via Proposition~\ref{eq:alg}.
\begin{algorithm}[h]
\begin{algorithmic} 
\caption{Predicting a binary matrix with side information in the inductive setting. \label{alg:inductive}}
\renewcommand{\algorithmicrequire}{\textbf{Parameters:}} 
\REQUIRE Learning rate: $0<\lr$\, quasi-dimension estimate: $1\le \scp$, margin estimate: $0 <\marh\leq 1$, non-conservative flag $[\mbox{\sc non-conservative}]\in \{0,1\}$ and side-information kernels $\RNfunc:\mathcal{I}\times\mathcal{I}\rightarrow\Re$, $\CNfunc:\mathcal{J}\times\mathcal{J}\rightarrow\Re$,  
with $\IRRN := \max_{i\in\mathcal{I}} \RNfunc(i,i)$ and $\IRCN := \max_{j\in\mathcal{J}} \CNfunc(j,j)$,
and maximum distinct rows $m$ and columns $n$, where $m+n \geq 3$.
\\  \vspace{.1truecm}
\renewcommand{\algorithmicrequire}{\textbf{Initialization:}}   \vspace{.1truecm} \vspace{.1truecm}
\REQUIRE $\nmset \leftarrow \emptyset\,,\uset \leftarrow \emptyset\,, \cI^1 \leftarrow \emptyset\,, \,\cJ^1 \leftarrow \emptyset\, \,.$   \vspace{.1truecm}
\renewcommand{\algorithmicrequire}{\textbf{For}}
\REQUIRE $t =1,\dots,T$  \vspace{.1truecm}
\STATE $\bullet$ Receive pair $(i_t,j_t) \in \mathcal{I} \times \mathcal{J}.$ \\ \vspace{.1truecm}
\STATE $\bullet$ Define 
\begin{align*}\label{???}
&\quad(\RN^t)^+   := (\RNfunc(i_r,i_s))_{r,s\in \cI^t\cup \{i_t\}}\,; \quad (\CN^t)^+  := (\CNfunc(j_r,j_s))_{r,s\in \cJ^t\cup \{j_t\}}\,,\\
& \quad \XT(s) := \con{\frac{(\sqrt{(\RN^t)^+})\be^{i_s}}{\sqrt{2\IRRN}}}{\frac{(\sqrt{(\CN^t)^+})\be^{j_s}}{\sqrt{2\IRCN}}} \con{\frac{(\sqrt{(\RN^t)^+})\be^{i_s}}{\sqrt{2\IRRN}}}{\frac{(\sqrt{(\CN^t)^+})\be^{j_s}}{\sqrt{2\IRCN}}}^{\trans}\,, \\
&\quad\log(\wem{t}) \leftarrow \log\left(\frac{\scp}{m+n}\right) \id^{|\cI^t|+|\cJ^t| +2} + \sum_{s\in\uset}  \lr\y{s}  \XT(s)\,.
\end{align*}
\vspace{-.13in}
\STATE $\bullet$  Predict 
\begin{equation*} \Yrv \sim \mbox{\sc Uniform}(-\marh,\marh) \!\times\! [\mbox{\sc non-conservative}]\,;\ybt \leftarrow \tr{\wem{t}\XT}-1 \,;\quad\yht \leftarrow \sign(\ybt-\Yrv)\,.
\end{equation*}\vspace{-.16in}
\STATE $\bullet$ Receive label $\y{t} \in \{-1,1\}$\,.\vspace{.1truecm}
\STATE $\bullet$ If $y_t \ne \yht$ then $\nmset \leftarrow \nmset \cup \{t\}.$
\STATE $\bullet$ If $y_t\ybt < \marh \times [\mbox{\sc non-conservative}] $ then 
\begin{equation*}
\uset \leftarrow \uset \cup \{t\}\,,\ \ \cI^{t+1} \leftarrow \cI^{t} \cup \{i_t\}, \text{ and } \cJ^{t+1} \leftarrow \cJ^{t} \cup \{j_t\}\,.
\end{equation*}\vspace{-.16in}
\STATE $\bullet$  Else $\cI^{t+1} \leftarrow \cI^{t}$ and $\cJ^{t+1} \leftarrow \cJ^{t}$\,.
\end{algorithmic}
\end{algorithm}
The intuition behind the algorithm is that, although we cannot efficiently embed the row and column kernel functions $\RNfunc$ and $\CNfunc$ as matrices since they are potentially infinite-dimensional, we may instead work with the embedding corresponding to the currently observed rows and columns, recompute the embedding on a per-trial basis and then ``replay'' all re-embedded past examples to create the current hypothesis matrix.

The computational complexity of the inductive algorithm exceeds that of the transductive algorithm.
For the following analysis, assume $m \in \Theta(n)$.  On every trial (with an update), Algorithm~\ref{alg:base} requires the computation of the SVD of an $n \times n$ matrix and thus requires $\cO(n^3)$ time.  On the other hand,  for every trial (with an update) in Algorithm~\ref{alg:inductive}, the complexity is instead dominated by the sum of up to $m n$ (i.e., in the regret setting we can collapse terms from multiple observations of the same matrix entry) matrices of size up to $(m+n) \times (m+n)$ and thus has a per-trial complexity $\cO(n^4)$.  
The following is our proposition of equivalency, proven in Appendix~\ref{sec:append_inductive}.
\begin{proposition}\label{eq:alg}
The inductive and transductive algorithms are equivalent up to $\RRN$ and $\RCN$.  Without loss of generality assume $\cI^{T+1} \subseteq [m]$ and  $\cJ^{T+1} \subseteq [n]$. 
Define $\RN := ((\RNfunc(i',i''))_{i',i''\in [m]})^{+}$ and $\CN := ((\CNfunc(j',j''))_{j',j''\in [n]})^{+}$. Assume that for the transductive algorithm, the matrices $\RN$ and $\CN$ are given whereas for the inductive algorithm, only the kernel functions $\RNfunc$ and $\CNfunc$ are provided. Then, if $\IRRN=\RRN$ and $\IRCN = \RCN$, and if the algorithms receive the same label and index sequences, then the predictions  of the algorithms are the same.
\end{proposition}
Thus, the only case when the algorithms are different is when $\IRRN\ne\RRN$ or $\IRCN \ne \RCN$.  This is a minor inequivalency, as the only resultant difference is in the term $\pdim$.  Alternatively, if one uses a normalized kernel such as the Gaussian, then $\IRRN =\RRN=1$.  In the following subsection, we describe a scenario where the quasi-dimension bound $\pdimb$ scales quadratically with the number of distinct factors. 
\vspace{-0.1in}
\subsection{Side information in $[-r,r]^d$}\label{sec:boxstory}
In the following, we show an example for predicting a matrix $\U\in\bikl$ such that for online side information in $[-r,r]^d$ that is well-separated into {\em boxes}, there exists a kernel for which the quasi-dimension grows no more than quadratically with the number of latent factors (but exponentially with the dimension $d$).  For simplicity we use the {\em min} kernel, which approximates functions by linear interpolation. In practice, we speculate that similar results may be proven for other universal kernels, but the analysis with the min kernel has the advantage of simplicity.

In the previous section, with the idealized graph-based side information, one may be dissatisfied as the 
skeleton of the latent structure is essentially encoded into $\RN (\CN)$.   In the inductive setting, the side information is instead revealed in an online fashion.  If such side information may be separated into distinct clusters restrospectively, we will be able to bound $\pdimb \in \cO(k^2 + \ell^2)$.  In this example, we receive a row and column vector $\imath_t,\jmath_t \in [-r,r]^d \times [-r',r']^{d'}$ on each trial; these vectors will be the indices to our row and column kernels, and for simplicity we set $r=r'$ and $d=d'$.  
\vspace{-0.1in}
\subsection*{Bounding $\pdimb$ for the {\em min} kernel.}
Define the transformation $s(\bxx):=\frac{r-1}{2r} \bxx + \frac{r+1}{2} $.
and the {\em min} kernel $\cK : [0, r]^d \times [0 , r]^d\into \Re$ as $\cK(\bxx,\bt) := \prod_{i=1}^d \min(x_i,t_i)$. Also define $\delta(S_1,\ldots,S_k) := \min_{1\le i < j \le k} \min_{\bxx\in S_i, \bxx' \in S_j} \norm{\bxx - \bxx'}_{\infty}$.
A {\em box} in $\Re^d$ is a set $\{\bxx : a_i \le x_i \le b_i, i\in [d]\}$ defined by a pair of vectors $\ba,\bb\in\Re^d$. 
\begin{proposition}\label{lem:min_kernel_norm}
Given $k$ boxes $S_1,\ldots,S_k \subset [-r,r]^d$, $r\ge 2$, $\delta^* = \min \left(2,  \frac{1}{4}\delta(S_1,\ldots, S_k) \right)$,  and  $\bxx_1,\ldots,\bxx_m\in\cup_{i=1}^k S_i$,  if $ \bR = ([\bxx_i\in S_j])_{i\in[m],j\in [k]}$ and  $\bK= (\cK(s(\bxx_i),s(\bxx_j)))_{i,j\in [m]}$ 
then $\trc(\bR^\trans \bK^{-1} \bR)  \le k \left(\frac{4}{\delta^*}\right)^d$.
\end{proposition}
Recall the bound (see~\eqref{eq:defpdimb}) on the quasi-dimension for a matrix $\U\in\bikl$,  where we have $\pdim \le
\pdimb= 
k\trace{\bR^{\trans} \bM \bR}\RRN+\ell\trace{\bC^{\trans} \bN \bC}\RCN
$ for positive definite matrices. If we assume that the side information on the rows (columns) lies in $ [-r,r]^d$, then $\RRN\le \IRRN \le r^d$ ($\RCN\le \IRCN \le r^d$) for the min kernel.  
Thus by applying the above proposition separately for the rows and columns and substituting into~\eqref{eq:defpdimb}, we have that 
\[ \pdim \le
\pdimb= 
k^2 \left(\frac{4r}{\delta^*}\right)^d +\ell^2 \left(\frac{4r}{\delta^*}\right)^d.\]
We then observe that for this example,with an optimal tuning and well-separated side information on the rows and columns, the mistake bound for a $(k,\ell)$-biclustered matrix in the inductive setting is of $\cOT(\min(k,\ell)\max(k,\ell)^2)$.  
However, our best lower bound in terms of $k$ and $\ell$ is just $k\ell$, as in the transductive setting. 
An open problem is to resolve this gap.

\section{Acknowledgements}
We would like to thank Robin Hirsch for valuable discussions.
This research was sponsored by the U.S. Army Research Laboratory and the U.K. Ministry of Defence under Agreement Number W911NF-16-3-0001. The views and conclusions contained in this document are those of the authors and should not be interpreted as representing the official policies, either expressed or implied, of the U.S. Army Research Laboratory, the U.S. Government, the U.K. Ministry of Defence or the U.K. Government. The U.S. and U.K. Governments are authorized to reproduce and distribute reprints for Government purposes notwithstanding any copyright notation hereon.  This research was further supported by the Engineering and Physical Sciences Research Council [grant number EP/L015242/1].

\clearpage
{
\bibliographystyle{unsrt}
\bibliography{mp19}
}
\newpage
\appendix

\section{Synthetic Experiments}\label{sec:experiments}
\ifviz
\begin{wrapfigure}{R}{0.45\textwidth} 
     \input{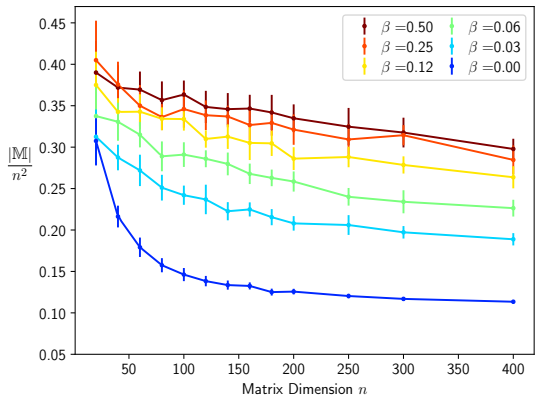}
    \caption{Error rates for predicting a noisy $(9,9)$-biclustered matrix with side information.} 
    \label{fig:results}
\end{wrapfigure}
\else

\begin{wrapfigure}{R}{0.45\textwidth} 
\vspace{-0.5cm}
	\centering
     \includegraphics[trim=5 5 18 30,height=.9\figureheight, width=.9\figurewidth,clip]{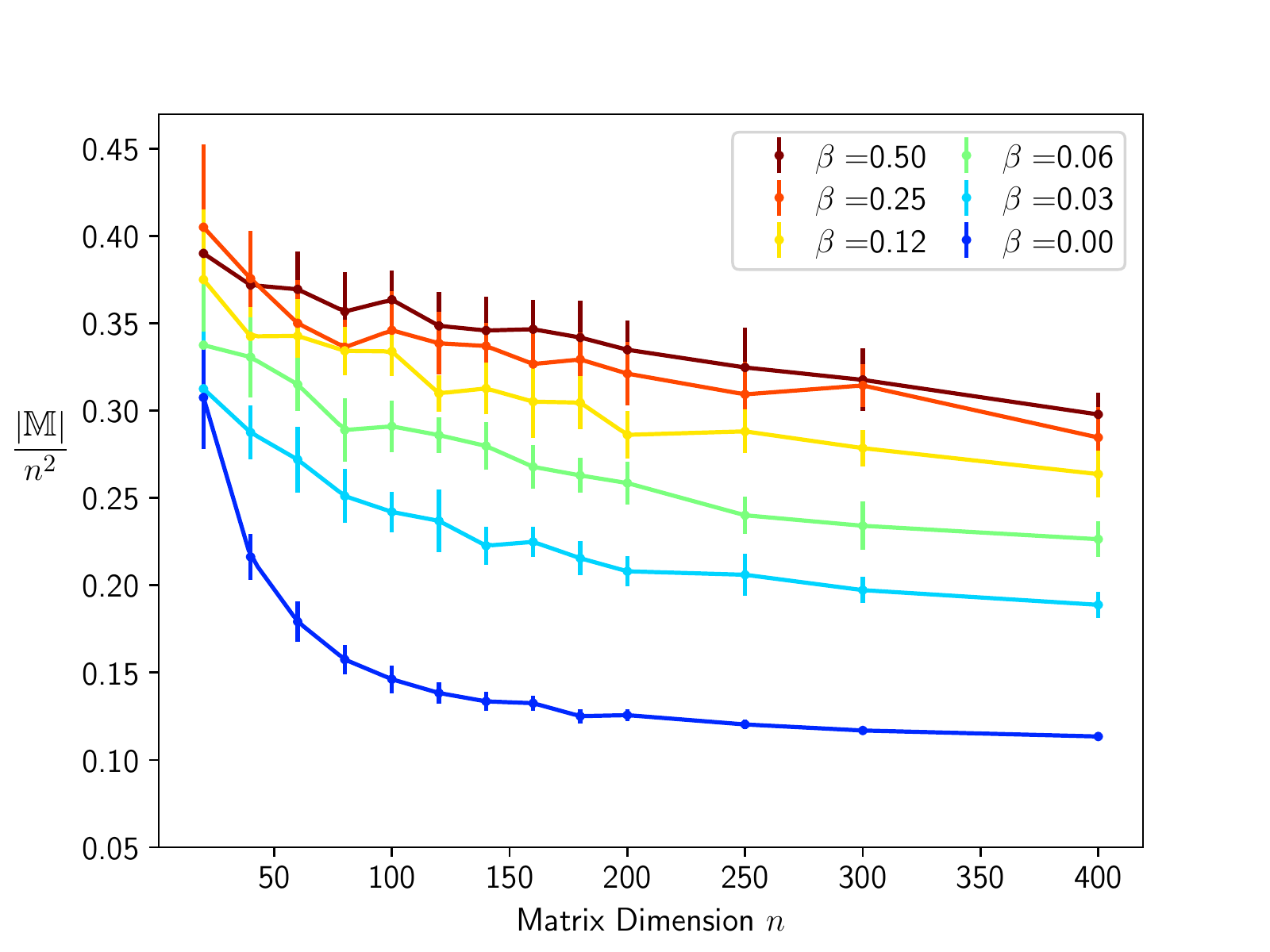}
     \vspace{-0.5cm}
   \caption{Error rates for predicting a noisy $(9,9)$-biclustered matrix with side information.} 
    \label{fig:results}
\end{wrapfigure}
\vspace{-0.1in}
\fi
To illustrate the algorithm's performance, synthetic experiments were performed in the transductive setting with graph side information. In particular, we took $\bU$ to be  randomly generated square (9,9)-biclustered matrices with i.i.d. noise. A visualization of a noise-free example matrix can be found in Figure~\ref{fig:exbcm}. The noise process flipped the label of each matrix entry independently with probability $p=0.10$. 
The side information on the rows and columns were represented by PDLaplacian matrices, for which the underlying graphs were constructed in the manner described in Section~\ref{section:graph_side_info}. Varying levels of side information noise $\beta\in [0.0,0.5]$ were applied. This was introduced by considering every pair of vertices independently from the constructed graph and flipping the state between {\sc edge}/{\sc not-edge} with probability~$\beta$.  A final step is added to ensure the graph is connected.  In this step a random pair of components is connected by a random edge, recursively. The process terminates when the graph is connected.

The parameters were chosen so that the expected regret bound in Theorem~\ref{thm:base} would apply to our experimental setting. We use the quasi-dimension upper bound $\upD := \qDBF(\U) =2 \trc( \bR^{\trans} \RN \bR)\RRN +  2 \trc( \bC^{\trans} \CN \bC)\RCN+4k $, as developed in  Theorem~\ref{thm:bndpdim} for PDLaplacians. The learning rate was set as $\lr = \sqrt{\frac{\upD \log(2n) }{2T}} $. Since each run of the algorithm consisted of predicting all $n^2$ matrix entries sampled uniformly at random without replacement, we set $T=n^2$. As for the margin estimate, due to the requirement that $\marh \leq 1/\maxnorm{\U}$, a suitable value can be extracted from Equation~\eqref{eq:mclbl}, giving $\marh = 1/ \sqrt{k}$.

The per trial mistake rate is shown in Fig.~\ref{fig:results} for matrix dimension $n=20,\ldots,400$, where each data point is averaged over 10 runs.
We observe that for random side information $\beta=0.5$, the term $\upD$ could lead to a bound which is vacuous (for small $n$), however, the algorithm's error rate was in the range of $[0.30, 0.45]$, being well below chance.
With ideal side information, $\beta=0.0$, the performance improved drastically, as suggested by the bounds, to an error rate in $[0.10,0.35]$. Observe that since there is 10\% label noise for all values of $\beta$, the curves are converging to an online mistake rate of 10\%. 
The data points for the plot can be found below.

\subsection{Data Points}\label{app:experiments}
The data points used to generate Fig.~\ref{fig:results} can be found in Table~\ref{table:raw_data}.  
\renewcommand{\arraystretch}{1.2}
\begin{table}
\centering
\caption{Data points used for Fig.~\ref{fig:results}.}
\begin{tabular}{@{}l l l l l l l@{}}
\toprule
\multirow{2}{*}{Matrix Dimensions $n$}& \multicolumn{6}{c}{Noise}\\ \cmidrule{2-7}
& 0.50&0.25&0.125&0.0625&0.03125&0.00\\ \midrule
20&0.39$\pm$ 0.04&0.4$\pm$ 0.05&0.38$\pm$ 0.04&0.34$\pm$ 0.04&0.31$\pm$ 0.03&0.31$\pm$ 0.03\\
40&0.37$\pm$ 0.03&0.38$\pm$ 0.03&0.34$\pm$ 0.02&0.33$\pm$ 0.02&0.29$\pm$ 0.02&0.22$\pm$ 0.01\\
60&0.37$\pm$ 0.02&0.35$\pm$ 0.02&0.34$\pm$ 0.02&0.32$\pm$ 0.02&0.27$\pm$ 0.02&0.18$\pm$ 0.01\\
80&0.36$\pm$ 0.02&0.34$\pm$ 0.02&0.33$\pm$ 0.01&0.29$\pm$ 0.02&0.25$\pm$ 0.02&0.16$\pm$ 0.01\\
100&0.36$\pm$ 0.02&0.35$\pm$ 0.02&0.33$\pm$ 0.01&0.29$\pm$ 0.01&0.24$\pm$ 0.01&0.15$\pm$ 0.01\\
120&0.35$\pm$ 0.02&0.34$\pm$ 0.02&0.31$\pm$ 0.01&0.29$\pm$ 0.01&0.24$\pm$ 0.02&0.14$\pm$ 0.01\\
140&0.35$\pm$ 0.02&0.34$\pm$ 0.01&0.31$\pm$ 0.01&0.28$\pm$ 0.01&0.22$\pm$ 0.01&0.13$\pm$ 0.01\\
160&0.35$\pm$ 0.02&0.33$\pm$ 0.02&0.31$\pm$ 0.02&0.27$\pm$ 0.01&0.22$\pm$ 0.01&0.13$\pm$ 0.0\\
180&0.34$\pm$ 0.02&0.33$\pm$ 0.02&0.3$\pm$ 0.02&0.26$\pm$ 0.01&0.22$\pm$ 0.01&0.13$\pm$ 0.0\\
200&0.33$\pm$ 0.02&0.32$\pm$ 0.02&0.29$\pm$ 0.01&0.26$\pm$ 0.01&0.21$\pm$ 0.01&0.13$\pm$ 0.0\\
250&0.32$\pm$ 0.02&0.31$\pm$ 0.02&0.29$\pm$ 0.01&0.24$\pm$ 0.01&0.21$\pm$ 0.01&0.12$\pm$ 0.0\\
300&0.32$\pm$ 0.02&0.31$\pm$ 0.01&0.28$\pm$ 0.01&0.23$\pm$ 0.01&0.2$\pm$ 0.01&0.12$\pm$ 0.0\\
400&0.3$\pm$ 0.01&0.28$\pm$ 0.02&0.26$\pm$ 0.01&0.23$\pm$ 0.01&0.19$\pm$ 0.01&0.11$\pm$ 0.0\\ \bottomrule
\end{tabular}
\label{table:raw_data}
\end{table}

\section{Proof of Theorem~\ref{thm:base}}

The proof of Theorem~\ref{thm:base} is organized as follows. We start with the required preliminaries in Subsection~\ref{ap:preliminaries}, and then proceed to prove the regret statement of the theorem, given by Equation~\eqref{eq:baseregret}, in Subsection~\ref{ap:regretmw}. Finally, in Subsection~\ref{ap:mwthm1}, we provide a proof for the mistake bound in the realizable case, as stated in Equation~\eqref{eq:basereal}. 
\subsection{Preliminaries for Proof}\label{ap:preliminaries}

Suppose we have $\RN$, $\CN$ and $\bU$ as in Theorem~\ref{thm:base}. Instead of working with $\bU$ directly, we shall work with an embedding of this matrix. We have different treatments for the embedding of $\bU$ in the two parts of the proof. For Subsection~\ref{ap:regretmw}, let $\Up \in \Re^{m\times n}$ be such that $\Up = \marh \bU$. Following from the assumption that $\maxnorm{\bU} \leq \frac{1}{\gamma}$, there exist row-normalized matrices $\nP \in \Re^{n \times d}$ and $\nQ \in \Re^{m \times d}$ that give $\Up = \nP \nQ^\trans$. In Subsection~\ref{ap:mwthm1}, however, we have that $\Up = \marh \argmin \limits_{\bV \in \SP^1(\bU)}  \qD(\bV)$. In this section we assume $\mc(\bU) \leq \frac{1}{\gamma}$, which also guarantees a decomposition in terms of the row-normalized matrices $\nP$ and $\nQ$ such that $\Up = \nP \nQ^\trans$.

Let us then define the quasi-dimension with respect to a specific factorization \[\pdim := \RRN\trc\left(\nP^{\trans}\RN\nP\right) 
+
\RCN\trc\left(\nQ^{\trans}\CN\nQ\right).\] Note that in general for the case that $\Up = \marh \bU$, $\pdim \geq \qD(\bU)$, whereas for the case that $\Up \in \marh \SP^1(\bU)$, $\pdim \geq \min\limits_{\bV \in \SP^1(\bU)}  \qD(\bV).$  We proceed with the proof assuming that $(\nP,\nQ)$ is the optimal factorization for a given $\Up$. That is, for $\Up = \marh \bU$, we have that $(\nP,\nQ)$ is the factorization that satisfies $\pdim= \qD(\bU)$, and for $\Up = \marh \argmin \limits_{\bV \in \SP^1(\bU)}  \qD(\bV)$, $(\nP,\nQ)$ satisfies $\pdim = \min\limits_{\bV \in \SP^1(\bU)}  \qD(\bV)$ . 

Next, we define $\mat$, which is a positive semidefinite matrix, used as an embedding for $\Up$ in the analysis of the algorithm. Its exact relationship with $\Up$ is shown in Lemma~\ref{lem:stuttS}.
\begin{definition}\label{def:imat}
Define the $(m+n)\times (m+n)$ matrix $\imat$ as
\begin{equation}\label{eq:imatdef}
\imat := \begin{pmatrix}
\sqrt{\RRN}\SRN\nP \\
\sqrt{\RCN}\SCN\nQ
\end{pmatrix}\,.
\end{equation}
and
 construct $\mat$ as,
\[
\mat:=\imat\imat^{\trans} = 
\begin{pmatrix}
\RRN \SRN \nP \nP^\trans \SRN & \sqrt{\RRN\RCN} \SRN \nP \nQ^\trans \SCN \\
\sqrt{\RRN\RCN} \SCN \nQ \nP^\trans \SRN & \RCN \SCN \nQ \nQ^\trans \SCN
\end{pmatrix}\,.
\]

 \begin{lemma}\label{lem:stuttS}
For all trials $t\in [T]$, 
\[
\Up_{i_t j_t}=\tr{\mat{\XT}}-1
\] 
where $\mat$ is as constructed from Definition~\ref{def:imat}. 
\end{lemma}
\begin{proof}
We have:
\begin{align}
\tr{\mat{\XT}}&={\left(\bx{t}\right)}^\trans\mat\bx{t}\\
&=\left(\bx{t}\right)^\trans\imat\imat^{\trans}\bx{t}\\
&=\ini{{(\bx{t})}^{\trans}\imat}^2\,. \label{eq:nov}
\end{align}
Recall that 
\[
 \bx{t} = \con{\frac{\psm\eit}{\sqrt{2\RRN}}}{\frac{\ppsm\ejt}{\sqrt{2\RCN}}} \text{ and } \imat = 
 \begin{pmatrix}
\sqrt{\RRN}\SRN\nP \\
\sqrt{\RCN}\SCN\nQ
\end{pmatrix}
\]

Hence,
\begin{align}
(\bx{t})^{\trans}\imat 
& = \frac{(\psm\eit)^\trans}{\sqrt{2\RRN}} \sqrt{\RRN}\SRN\nP 
+ \frac{(\ppsm\ejt)^\trans}{\sqrt{2\RCN}}\sqrt{\RCN}\SCN\nQ \notag \\
& = \frac{1}{\sqrt{2}} (\eit)^\trans\psm\SRN\nP 
+ \frac{1}{\sqrt{2}}(\ejt)^\trans\ppsm\SCN\nQ \notag\\
& = \frac{1}{\sqrt{2}}(\nP_{i_t} + \nQ_{j_t}) \label{eq:mv}
\end{align}
Thus substituting~\eqref{eq:mv} into~\eqref{eq:nov} gives,
\begin{align*}
 \tr{\mat{\XT}}&=\frac{1}{2}\ini{\nP_{i_t}+\nQ_{j_t}}^2\\
 &=\frac{1}{2}\left(\ini{\nP_{i_t}}^2+2\dotp{\nP_{i},\nQ_{j_t}}+\ini{\nQ_{j_t}}^2\right)\\
  &=\left(1+\dotp{\nP_{i_t},\nQ_{j_t}}\right)\\
  &=1+\Up_{i_t j_t}.
\end{align*}
\end{proof}
\end{definition}
In the subsequent proofs, we will also need to make use of the following facts.
\begin{lemma}\label{lem:qexactS}  For $\mat$ as defined in Definition~\ref{def:imat}, we have that, 
\begin{equation}\label{eq:qexact}
\trc(\mat) = 
\pdim \,.
\end{equation}
\end{lemma}
\begin{proof}
\begin{align*}
\trc(\mat) & = \trc(\imat \imat^{\trans}) 
 = \trc\left(
\begin{pmatrix}
\sqrt{\RRN}\SRN\nP \\
\sqrt{\RCN}\SCN\nQ
\end{pmatrix}
\begin{pmatrix}
\sqrt{\RRN}\SRN\nP \
\sqrt{\RCN}\SCN\nQ
\end{pmatrix}^{\trans}
\right) \\
& = 
\RRN\trc\left(\SRN\nP\nP^{\trans}\SRN^{\trans}\right) 
+
\RCN\trc\left(\SCN\nQ\nQ^{\trans}\SCN^{\trans}\right) \\
&= \RRN\trc\left(\nP^{\trans}\RN\nP\right) 
+
\RCN\trc\left(\nQ^{\trans}\CN\nQ\right)  \\
&=\pdim
\end{align*}
\end{proof}

\begin{lemma}\label{eigenx}
For all trials $t$, all eigenvalues of $\XT$ are in $[0,1]$.  
\end{lemma}

\begin{proof}
Recall from~\eqref{eq:bbxdef} that 
\[ \trc(\XT)=\trc(\bx{t}(\bx{t})^\trans) = \con{\frac{\psm\eit}{\sqrt{2\RRN}}}{\frac{\ppsm\ejt}{\sqrt{2\RCN}}}^\trans \con{\frac{\psm\eit}{\sqrt{2\RRN}}}{\frac{\ppsm\ejt}{\sqrt{2\RCN}}}.  
\]
Hence 
\[
\ini{\bx{t}}^2 = \ini{\frac{\psm\eit}{\sqrt{2\RRN}}}^2 + \ini{\frac{\ppsm\ejt}{\sqrt{2\RCN}}}^2
\] and then bounding the first term on the right hand side gives,
\[
\ini{\frac{\psm\eit}{\sqrt{2\RRN}}}^2=\frac{1}{2\RRN}\left(\eit\right)^\trans(\psm)^\trans\psm\eit
\le\frac{1}{2\RRN}\max_{i\in[m]}\left(\eit\right)^\trans\RN^+\eit
=\frac{1}{2}\,.
\]
The argument for the second term is parallel.  Therefore since it is shown that the trace of $\XT$ is bounded by 1 and that $\XT$ is positive definite, this implies that all eigenvalues of $\XT$ are in $[0,1]$. 
\end{proof}

Next, we introduce the following quantity, which plays a central role in the amortized analysis of our algorithm.  
\begin{definition}{\label{def:qrel}} The quantum relative entropy of symmetric positive semidefinite square matrices $\boldsymbol{A}$ and $\boldsymbol{B}$ is
\begin{equation*}
\rent{\bA}{\bB}:=\trace{\bA\log({\bA})-\bA\log({\bB})+\bB-\bA}.
\end{equation*}
\end{definition}


An important result that will be used in the subsequent subsections is the well known Golden-Thompson Inequality, whose proof can be found, for example, in \cite{Bhatia1997}.
\begin{lemma}\label{ptl1}
For any symmetric matrices $\bs{A}$ and $\bs{B}$ we have,
\[
\trc(\exp(\bs{A}+\bs{B}))\leq\trc(\exp(\bs{A})\exp(\bs{B}))\,.
\]
\end{lemma}

\subsection{Proof for the Regret Statement}\label{ap:regretmw}

In this subsection, we prove the regret bound as presented in Theorem~\ref{thm:base}, which holds for Algorithm~\ref{alg:base} with {\bf non-conservative} updates. 
To do so, we first derive a regret bound in terms of the hinge loss with the deterministic prediction $\ybt$. We then convert this to an expected regret bound in terms of the 0-1 loss for the random variable $\yht$. 

Regret bounds for the MEG algorithm were originally proven in~\cite{tsuda2005matrix}. However, that analysis leads to a $\trc(\tilde{\bU})$ dependence, whereas we derive a $\sqrt{\trc(\tilde{\bU})}$ scaling, for our more restrictive setting. Regret bounds with such scaling for  linear classification in the vector case have been previously given in~\cite{Sabato2015} (which themselves are generalisations of the bounds from Littlestone~\cite{litt88} for learning $k$-literal disjunctions with $\cO(k \log n)$ mistakes). However, to our knowledge, no such regret bounds for MEG are present in the literature for the matrix case. Our proof uses an amortized analysis of the quantum relative entropy, followed by an application of the matricized results in~\cite{Sabato2015}.

In the following, we define the hinge loss as $\hl_{\mar}(y,\bar{y}) := \frac{1}{\mar} [\mar - y\bar{y}]_+$. We define 
\begin{equation}\label{eq:defH}
\bH^t:= \nabla_{\bWtilde} \marh h_\marh(y_t,\ybt)\,,
\end{equation}
 where $\nabla$ denotes the subgradient and where $\ybt$ is as defined in Algorithm~\ref{alg:base}. When $ y_t \ybt = \marh$, we will only consider the specific subgradient $\bH^t = \bzero$.

\begin{lemma}
\label{lemma:bZ_t}
For all $t\in [T]$,
 \[ \bH^t =  -y_t\XT\, \left[\marh > y_t\left(\tr{\wem{t}\XT}-1\right) 
\right]\,.\]
\end{lemma}
\begin{proof}
Recalling the definition of $\bH^t:= \nabla_{\bWtilde}\marh h_{\marh}(y_t,\ybt)$, observe that when $\marh > y_t\left(\tr{\wem{t}\XT}-1\right) $, we have
\begin{equation}
    \nabla_{\bWtilde}\marh h_{\marh}(y_t,\ybt) = \nabla_{\bWtilde} \left[\marh - y_t\left(\tr{\wem{t}\XT}-1\right)\right]_+ = -y_t(\XT)^\trans =- y_t \XT,
\end{equation}
where we used the fact that  $\nabla_{\bA} \tr{\bA\bB}= \bB^\trans$. 
In the case that $\marh \leq y_t\left(\tr{\wem{t}\XT}-1\right) $, \[  \nabla_{\bWtilde} \marh h_{\marh}(y_t,\ybt) = \bzero.\]
\end{proof}

\begin{lemma}
\label{lemma:psd_exp}
For matrix $\bA \in \bS^d$ with eigenvalues no less than -1, 
\[ \id - \bA + \bA^2 - \exp(-\bA) \succeq \bzero\,.
\]
\end{lemma}
\begin{proof}
Let $\bB := \id - \bA + \bA^2 - \exp(-\bA)$.
Observing that $\bA$, $\bA^2$ and $\exp(-\bA)$ share the same set of eigenvectors,
\begin{equation*}
     \id -\bA + \bA^2 - \exp(-\bA) = \bU(\id -\bLam+\bLam^2 - \exp(-\bLam)) \bU^\trans ,  
\end{equation*}
where $\bU$ is the orthogonal matrix and $\bLam$ is the diagonal matrix in the eigendecomposition of $\bA$. Therefore, each eigenvalue $\lambda_{\bB,i}$ of the resulting matrix $\bB$ can be written in terms of an eigenvalue $\lambda_{\bB,i}$ of matrix $\bA$ for all $i \in [d]$,
\begin{equation*}
    \lambda_{\bB,i} = 1 - \lambda_{\bA,i} +\lambda_{\bA,i}^2 - \exp(-\lambda_{\bA,i}).
\end{equation*}

\begin{lemma}\label{lem:diff_psd}
For any matrix $\bA \in \bS^d_{++}$ and any two matrices $\bB,\bC \in \bS^d$, $\bB \preceq \bC$ implies $\tr{\bA\bB} \leq \tr{\bA\bC}$.
\end{lemma}
\begin{proof}
This follows a parallel argument to the proof for~\cite[Lemma 2.2]{tsuda2005matrix}.
\end{proof}

The positive semidefinite criterion requires that all eigenvalues be non-negative, so that $\lambda_{\bB,i} \geq 0$. This inequality holds true for $\lambda_{\bA,i} \geq -1$.
\end{proof}

\begin{lemma}\label{malS_cons}
For all trials $t\in [T]$ in Algorithm~\ref{alg:base}, we have for $\lrreg \in (0,1]$:
\begin{equation}
\rent{\mat}{\wem{t}}-\rent{\mat}{\wem{t+1}}
\geq \lrreg \left(\trc((\wem{t} - \mat) \bH^t)- \lrreg^2 \trc(\wem{t}(\bH^t)^2)\right)\,.
\end{equation}
\end{lemma}
\begin{proof}
We have:
\begin{align}
\rent{\mat}{\wem{t}}-\rent{\mat}{\wem{t+1}} &=\tr{\mat\log\wem{t+1}-\mat\log\wem{t}}+\tr{\wem{t}}-\tr{\wem{t+1}} \notag\\
\label{tseq1}& =-\lrreg \tr{\mat\bH^t}+\tr{\wem{t}}-\tr{e^{\log\wem{t}-\lrreg \bH^t}}\\
\label{tseq2}&\geq-\lrreg \tr{\mat\bH^t}+\tr{\wem{t}}-\tr{e^{\log\wem{t}}e^{-\lrreg \bH^t}}\\
&=-\lrreg \tr{\mat\bH^t}+\tr{\wem{t}\left(\id-e^{-\lrreg \bH^t}\right)}\notag\\
&\geq-\lrreg \tr{\mat\bH^t}+\tr{\wem{t}\left(\lrreg \bH^t - \lrreg^2 (\bH^t)^2\right)}\label{tseq3}
\end{align}
where Equation~\eqref{tseq1} comes from the update of the algorithm and Lemma~\ref{lemma:bZ_t}, Equation~\eqref{tseq2} comes from Lemma~\ref{ptl1} and Equation~\eqref{tseq3} comes from Lemmas~\ref{lemma:psd_exp} and~\ref{lem:diff_psd}.
\end{proof}


\begin{lemma}
\label{theorem:bound_with_Ztsquared}
For Algorithm~\ref{alg:base}, we have for $\lrreg \in (0,1]$:
\begin{equation}
\label{eq:Ztsquared1}
     \sum_{t=1}^T \tr{(\wem{t}-\mat)\bH^t} \leq  \frac{1}{\lrreg} \left( \tr{\mat\log\left(\frac{\mat (m+n)}{e\scp}\right)} + \scp\right) + \sum_{t=1}^T   \lrreg \tr{\wem{t}(\bH^t)^2}.
\end{equation}
Setting the additional assumptions $\scp \geq \tr{\mat} \geq 1$ and $m+n\geq3$ gives
\begin{equation}
\label{eq:Ztsquared2}
     \sum_{t=1}^T \tr{(\wem{t}-\mat)\bH^t} \leq  \frac{\scp}{\lrreg} \log\left(m+n\right) + \sum_{t=1}^T   \lrreg \tr{\wem{t} (\bH^t)^2}.
\end{equation}
\end{lemma}
\begin{proof} We start by proving Equation~\eqref{eq:Ztsquared1}.
Rearranging Lemma~\ref{malS_cons} and summing over $t$,
\begin{align}
\sum_{t=1}^T \left(\trc((\wem{t} - \mat) \bH^t)\right) 
&\leq \frac{1}{\lrreg} \left( \sum_{t=1}^T \rent{\mat}{\wem{t}}-\rent{\mat}{\wem{t+1}} +  \lrreg^2 \tr{\wem{t}(\bH^t)^2}\right) \notag \\
&\leq \frac{1}{\lrreg} \left(\rent{\mat}{\wem{1}}-\rent{\mat}{\wem{T+1}} + \sum_{t=1}^T   \lrreg^2 \tr{\wem{t}(\bH^t)^2}\right). \notag
\end{align}
Using the fact that $\rent{\mat}{\wem{T+1}} \geq 0$ and writing out $\rent{\mat}{\wem{1}}$, we then obtain Equation~\eqref{eq:Ztsquared1}.

To prove Equation~\eqref{eq:Ztsquared2}, we attempt to maximize the term $\tr{\mat\log\left(\frac{\mat (m+n)}{e\scp}\right)} + \scp$. Noting that $\trc\left(\mat \log (a \mat) \right) \leq \trc(\mat) \log \left( \trc(a\mat)\right)$ for $a\geq 0$, we then have \[ \tr{\mat\log\left(\frac{\mat (m+n)}{e\scp}\right)} + \scp \leq \trc(\mat) \log \left ( \frac{\trc(\mat) (m+n)}{e\scp} \right) + \scp. \] The upper bound in the above equation is convex in $\trc(\mat)$ and hence is maximized at either boundary $\{1, \scp \}$. Comparing the terms, we have \[ \tr{\mat\log\left(\frac{\mat (m+n)}{e\scp}\right)} + \scp \leq \log\left(\frac{m+n}{e\scp}\right)+ \scp\] for $\trc(\mat)=1$ and \[\tr{\mat\log\left(\frac{\mat (m+n)}{e\scp}\right)} + \scp \leq \scp \log(m+n)\] for $\trc(\mat) = \scp$. We then observe that given the assumptions, $\scp \log(m+n)$ maximizes, therefore giving the upper bound in Equation~\eqref{eq:Ztsquared2}.

\end{proof}

\begin{lemma}
\label{lemma: bound_Wt_Zt}
The following condition is satisfied for Algorithm~\ref{alg:base}:
\begin{equation}
\label{eq:loss_function_assumption}
    \tr{\wem{t} (\bH^t)^2} \leq \marh h_{\marh}(y_t,\ybt) +\marh + 1.
\end{equation}
\end{lemma}
\begin{proof}
The proof splits into two cases.

Case 1) $\marh \leq y_t\ybt = y_t \left(\tr{\wem{t}\XT}-1\right)$:

Observe that $\bH^t = \bzero$ due to Lemma~\ref{lemma:bZ_t}, giving $\tr{\wem{t} (\bH^t)^2} = 0.$ which demonstates~\eqref{eq:loss_function_assumption} in this case.

Case 2) $\marh > y_t\ybt = y_t\left(\tr{\wem{t}\XT}-1\right)$:\\
We have that
\begin{equation}\label{eq:upper_bound_Zt}
\tr{\wem{t} (\bH^t)^2} =\tr{\wem{t} (\XT)^2} \leq  \tr{\wem{t}\XT},
\end{equation}
where the first equality comes from the fact $\bH^t = -y_t\XT$ from Lemma~\ref{lemma:bZ_t} and
the second inequality comes from Lemma~\ref{lem:diff_psd} and the fact that $(\XT)^2 \preceq \XT$ due to Lemma~\ref{eigenx}.

We split case 2 into two further subcases.

Sub-case 1) $\tr{\wem{t}\XT} < \marh + 1$ (Prediction smaller than margin):

Since we have 
\[
\tr{\wem{t}\XT} <\marh h_{\marh}(y_t,\ybt)+  \marh + 1,
\]
lower bounding the L.H.S. by~\eqref{eq:upper_bound_Zt} demonstrates~\eqref{eq:loss_function_assumption}.

Sub-case 2) $\tr{\wem{t}\XT} \ge \marh + 1$ (Prediction larger than margin with mistake):

We have 
\[ \tr{\wem{t}\XT} \leq \left[\tr{\wem{t}\XT} + \marh - 1 \right]_+ - (\marh - 1) \leq  \left[\tr{\wem{t}\XT} + \marh - 1\right]_+ + (\marh + 1)\,.
\]
By the case 2 and sub-case 2 conditions we have that $y_t = -1$, with
\[
\marh h_{\marh}(-1,\ybt) =\left[ \marh+ \tr{\wem{t}\XT} -1 \right]_+.
\]
Thus we have 
\[
\tr{\wem{t}\XT} \leq \marh h_{\marh}(-1,\ybt)  + (\marh + 1)
\] and by lower bounding L.H.S. by~\eqref{eq:upper_bound_Zt} we demonstrate~\eqref{eq:loss_function_assumption} and thus the lemma.
\end{proof}

Now we are ready to introduce the regret bound in terms of the hinge loss for the deterministic $\ybt$.
\begin{lemma}\label{theorem1part1}
The hinge loss of Algorithm~\ref{alg:base} with parameters $\mar \in (0,1]$, $\scp \geq \pdim\geq 1$,
$\lrreg = \sqrt{\frac{\scp \log(m+n)}{2 T}}$, $T  \geq 2 \scp\log(m+n)$ and $m+n \geq 3$, is bounded by
\begin{equation}
\sum_{t\in [T]} \hlm(y_t,\ybt)  \le \sum_{t\in [T]} \hlm(y_t,\Up_{i_t j_t})+ \frac{4}{\mar} \sqrt{2 \scp \log(m+n) T} + \frac{4}{\mar} \scp \log(m+n)\,,
\end{equation}
where $\Up$, $\pdim$ and their relationship are defined in the preliminaries of the proof. 
\end{lemma}
\begin{proof}
Substituting for $\ybt$ gives,
\[h_\marh(y_t,\ybt) = \frac{1}{\marh}[\marh - y_t (\tr{\wem{t}\XT} - 1)]_+\,.\]
Lemma~\ref{lem:stuttS} gives, 
\[h_\marh(y_t,\yut) = \frac{1}{\marh}[\marh - y_t (\tr{\mat\XT} - 1)]_+\,.\] 
Define
\[
f_t(\bZ) := \frac{1}{\marh}[\marh - y_t (\tr{\bZ\XT} - 1)]_+
\]
Since $h_\marh(y_t,\cdot)$ is convex and the fact that a convex function applied to a linear function is again convex we have that $f(\cdot)$ is convex.
We have
\begin{align}
    \sum_{t=1}^T\left(h_{\marh}(y_t,\ybt) -h_{\marh}(y_t,\yut)\right) & =\sum_{t=1}^T\left( f_t(\wem{t}) -f_t(\mat)\right) \notag \\ 
&  \leq \sum_{t=1}^T \tr{\left(\wem{t}-\mat\right)^{\trans} \nabla f_t(\bWtilde)} \label{eq:convex_hg} \\
&  = \sum_{t=1}^T \tr{\left(\wem{t}-\mat\right)^{\trans} \nabla_{\bWtilde} h_{\marh}(y_t,\ybt)} \notag \\
&=\frac{1}{\marh}\sum_{t=1}^{T} \tr{(\wem{t}-\mat)\bH^t} \label{eq:apply_defZt},
\end{align}
where~\eqref{eq:convex_hg} follows from the fact that $f(\bA) - f(\bB) \leq \trc((\bA-\bB)^\trans \nabla f(\bA))$ for a convex function~$f$ and~\eqref{eq:apply_defZt} comes from the definition of $\bH^t = \nabla_{\bWtilde} \marh h_\marh(y_t,\ybt)$ (see~\eqref{eq:defH}) and the fact that $\wem{t}$ and $\mat$ are symmetric.

Therefore, we only need an upper bound to $\sum_{t=1}^{T} \tr{(\wem{t}-\mat)\bH^t}$. 
We observe that $\lrreg \in \left(0,\frac{1}{2}\right]$ due to the definition of $\lrreg$ and the assumption on $T$.  Thus we can apply~\eqref{eq:Ztsquared2} and~\eqref{eq:loss_function_assumption} to obtain
\[
\sum_{t=1}^{T} \tr{(\wem{t}-\mat)\bH^t} \le  \frac{1}{\lrreg} \scp\log(m+n)  + \lrreg\marh \sum_{t=1}^{T} h_{\marh}(y_t,\ybt) + \lrreg   (1+\marh) T\,. 
\]
Substituing the above into~\eqref{eq:apply_defZt} gives,
\begin{align*}
\sum_{t=1}^{T} h_{\marh}(y_t,\ybt) &\leq \sum_{t=1}^T h_{\marh}(y_t,\yut) +\frac{1}{\marh}\left( \frac{1}{\lrreg} \scp\log(m+n)  + \lrreg\marh \sum_{t=1}^{T} h_{\marh}(y_t,\ybt) + \lrreg   (1+\marh) T \right) \\
& = \left(\frac{1}{1 - \lrreg} \right)\left (\sum_{t=1}^T h_{\marh}(y_t,\yut) + \frac{1}{\lrreg \marh}  \scp \log(m+n)  + \frac{\lrreg}{\marh}   (1+\marh) T\right) \\
&  \leq \left(\frac{1}{1 - \lrreg} \right)\left (\sum_{t=1}^T h_{\marh}(y_t,\yut) + \frac{1}{\lrreg \marh}  \scp \log(m+n) + \frac{2 \lrreg}{\marh}  T\right)\,,
\end{align*}
where the final inequality follows since $\marh\in (0,1]$.

We apply $(1/(1-x)) \leq 1+ 2x$ for $x\in [0,1/2]$ to obtain 
\begin{align*}
     \sum_{t=1}^{T} h_{\marh}(y_t,\ybt) &\leq \left(1 + 2 \lrreg \right)\left (\sum_{t=1}^T h_{\marh}(y_t,\yut) + \frac{1}{\lrreg\marh} \left( \scp\log(m+n) \right) + \frac{2 \lrreg}{\marh}   T\right)\\
     & = \overbrace{(1+2 \lr)\sum_{t=1}^T h_{\marh}(y_t,\yut)}^{(1)} + \overbrace{\frac{1}{\lr\marh} \scp \log(m+n)}^{(2)} + \overbrace{\frac{2\lr}{\marh} T}^{(3)}+ \overbrace{\frac{2}{\marh} \scp \log(m+n)}^{(4)} + \overbrace{\frac{4\lr^2}{\marh}T}^{(5)}
     \end{align*}
Then, substituting the value for $\lrreg$, 
\begin{multline*}
  \sum_{t=1}^{T} h_{\marh}(y_t,\ybt)  \leq \overbrace{\sum_{t=1}^T h_{\marh}(y_t,\yut)}^{(a)}+ \overbrace{\sqrt{\frac{2 \scp \log(m+n)}{T}} \sum_{t=1}^T h_{\marh}(y_t,\yut)}^{(b)} +\\ \overbrace{\frac{4}{\marh} \scp \log(m+n)}^{(c)} + \overbrace{\frac{2}{\marh} \sqrt{2\scp \log(m+n)  T}}^{(d)}\,
\end{multline*}
where $(1)=(a)+(b)$, $(2)+(3)=(d)$, and $(4)+(5)=(c)$.
Recalling $\yut=\marh\mat_{i_t,j_t}$ from the preliminaries of the proof, we then have $h_{\marh}(y_t,\yut) \leq 2 \leq \frac{2}{\marh}$, giving
\[
\sum_{t=1}^{T} h_{\marh}(y_t,\ybt) - \sum_{t=1}^T h_{\marh}(y_t,\yut) \leq \frac{4}{\marh} \sqrt{2 \scp\log(m+n)  T} + \frac{4}{\marh} \scp \log(m+n).
\]
\end{proof}

Before we can compute the desired regret bound in terms of the 0-1 loss, we first need to introduce the following relationships. 

\begin{lemma}\label{lem:hinge_EM_bound}
For $y_t \in \{-1,1\}$, $\ybt \in \Re$, $\Yrv\sim \mbox{\sc Uniform}(-\marh,\marh)$, $\marh \in (0,1]$ and $\yht:= \sign(\ybt - \Yrv)$,
\[2\mathbb{E}[ y_t \neq \yht]\leq h_\marh( y_t,\ybt). \]
\end{lemma}
\begin{proof}
We have
\[
p(\yht=1) = 
\begin{cases}
0 & \text{ if } \ybt \leq -\marh \\
\frac{1}{2} + \frac{\ybt}{2 \marh} & \text{ if } -\marh < \ybt \leq \marh\\
1 & \text{ if } \ybt > \marh
\end{cases}
\]
and
\[
p(\yht=-1) = 
\begin{cases}
1 & \text{ if } \ybt \leq -\marh \\
\frac{1}{2} - \frac{\ybt}{2 \marh} & \text{ if } -\marh < \ybt \leq \marh\\
0 & \text{ if } \ybt > \marh.
\end{cases}
\]
The possible cases are as follows.
\begin{enumerate}
\item If $|\ybt| < \marh$, $2 \mathbb{E}[ y_t \neq \yht] = h_\marh( y_t,\ybt)$. This is since if $y_t=1$, $\mathbb{E}[ y_t \neq \yht] = \frac{1}{2} - \frac{\ybt}{2\marh}$ and $h_\marh(y_t,\ybt) =\frac{1}{\marh} (\marh - \ybt) $. Similarly if $y_t=-1$, $\mathbb{E}[ y_t \neq \yht] = \frac{1}{2} + \frac{\ybt}{2\marh}$ and $h_\marh(y_t,\ybt) =\frac{1}{\marh} (\marh +\ybt) $.
\item If $|\ybt| \geq \marh$ and
$\mathbb{E}[y_t \neq \yht]=0$, then $h_\marh( y_t,\ybt)=  \frac{1}{\marh}[\marh - |\ybt|]_+ = 0$. 
\item If $|\ybt| \geq \marh$ and
$\mathbb{E}[ y_t \neq \yht]=1$,
$h_\marh( y_t,\ybt) = \frac{1}{\marh}[\marh + |\ybt|]_+ \geq \frac{2 \marh}{\marh} = 2 \mathbb{E}[ y_t \neq \yht] . $

\end{enumerate}
\end{proof}

\begin{lemma}\label{lem:bound_bU_marm}
Suppose we have $\bU$ as in Theorem~\ref{thm:base}. Recalling that $\Up = \marh \bU$ from the preliminaries of the proof, we have that \[h_{\marh}(y_t,\yut) \leq 2[ y_t \ne U_{i_t j_t}].\]
\end{lemma}
\begin{proof}
Recall that $\bU \in \{-1,1\}^{m\times n}$. The hinge loss is then given by
$h_\marh(y_t, \yut ) = \frac{1}{\marh} [\marh -y_t \yut ]_+. $
In the case that $y_t = \yut$,
$h_\marh(y_t, \yut ) = \frac{1}{\marh} [\marh - \marh ]_+  = 0. $
Otherwise,
$ h_\marh(y_t, \yut ) = \frac{1}{\marh} [\marh + \marh]_+  = 2. $
\end{proof}

We proceed by giving a sharper bound for Algorithm~\ref{alg:base} than is stated in Theorem~\ref{thm:base}. This, however, only holds under the additional assumption that $T  \geq 2 \upD \log(m+n)$. 

 \begin{theorem}\label{thm:base_bef_alg}
The expected regret of Algorithm~\ref{alg:base} with {\bf non-conservative} updates and parameters $\mar \in (0,1]$, $\upD \geq \qD(\bU)$ ,
$\lr = \sqrt{\frac{\upD \log(m+n) }{2 T}}$, p.d. matrices $\RN\in \SPDM^m$ and $\CN\in \SPDM^n$ and for $T  \geq 2 \upD \log(m+n)$  is bounded  by
\begin{equation}\label{eq:refinebaseregret}
\Exp[\nm] - \sum_{t\in [T]} [ y_t \ne U_{i_t j_t}] \le  \frac{2}{\mar} \sqrt{2 \upD\log(m+n) T} + \frac{2}{\mar} \upD \log(m+n) 
\end{equation}
for all $\bU\in \sm$ with $\maxnorm{\bU} \le 1/\mar$.
\end{theorem}
\begin{proof}
Starting from Lemma~\ref{theorem1part1}, we observe that we can apply Lemma~\ref{lem:hinge_EM_bound}, to bound the expected mistakes of the latter by the cumulative hinge loss of the former. Combining this with Lemma~\ref{lem:bound_bU_marm} then gives the desired regret bound in terms of the 0-1 loss. Note that $\upD \geq \pdim = \qD(\bU)$, where the inequality comes from Lemma~\ref{theorem1part1} and the equality is stated in the preliminaries of the proof, which assumes that $\maxnorm{\bU} \le 1/\mar$. 
\end{proof}

We will now prove the first part of Theorem 1.  We split into two cases,

Case 1) ($T  < 2 \upD \log(m+n)$):

we have that 
\[ \Exp[\nm] - \sum_{t\in [T]} [ y_t \ne U_{i_t j_t}] \le T < 2 \upD \log(m+n)\,\]
 for all $\lrreg>0$.

Case 2) ($T  \geq 2 \upD \log(m+n)$):

From~\eqref{eq:refinebaseregret} we have
\begin{align}
\Exp[\nm] - \sum_{t\in [T]} [ y_t \ne U_{i_t j_t}] & \le \frac{2}{\mar} \sqrt{2 \upD\log(m+n) T} + \frac{2}{\mar} \upD \log(m+n) \notag \\
 & =\frac{2}{\mar} \sqrt{2 \upD\log(m+n) T} + \frac{2}{\mar} \sqrt{(\upD \log(m+n))^2} \notag \\
 &\leq \frac{2}{\mar} \sqrt{2 \upD\log(m+n) T} + \frac{2}{\mar} \sqrt{\frac{1}{2}\upD \log(m+n) T } \label{eq:tbound} \\
 &=  \frac{2}{\mar} \left(\sqrt{2}+\sqrt{\frac{1}{2}}\right)\sqrt{\upD\log(m+n) T} \notag
\end{align}
where in Equation~\eqref{eq:tbound} we used the assumption on $T$.

Combining both cases, we have that the following holds for all $T$
\begin{align}
\Exp[\nm] - \sum_{t\in [T]} [ y_t \ne U_{i_t j_t}]& \leq  \frac{2}{\mar} \left(\sqrt{2}+\sqrt{\frac{1}{2}}\right)\sqrt{\upD\log(m+n) T} + \min(2 \upD \log(m+n), T) \notag \\
&= \frac{2}{\mar} \left(\sqrt{2}+\sqrt{\frac{1}{2}}\right)\sqrt{\upD\log(m+n) T} +\sqrt{ \min(2 \upD \log(m+n), T)^2} \notag\\
&\leq \frac{2}{\mar} \left(\sqrt{2}+\sqrt{\frac{1}{2}}\right)\sqrt{\upD\log(m+n) T} +\sqrt{ 2 \upD \log(m+n) T} \notag\\
&\leq \frac{2}{\mar} \left(\sqrt{2}+\sqrt{\frac{1}{2}}\right)\sqrt{\upD\log(m+n) T} + \frac{2}{\mar} \sqrt{ \frac{1}{2}\upD \log(m+n) T} \label{eq:gamtrick} \\
&=\frac{4}{\mar} \sqrt{ 2\upD \log(m+n) T} \notag
\end{align}
where we used the fact that $\frac{1}{\gamma} \geq 1$ in Equation~\eqref{eq:gamtrick}.  Thus we have demonstrated~\eqref{eq:baseregret} proving Theorem~\ref{thm:base}.
 \hfill $\blacksquare$
\subsection{Proof for the Realizable Case}\label{ap:mwthm1}

In this subsection, we prove the second part of Theorem~\ref{thm:base}. Recall from the theorem statement that $y_t = \bU_{i_t j_t}$ for all $t\in\nmset$,
$\mc(\bU)^{-1} \geq \mar$, $ \scp \geq \min\limits_{\bV \in \SP^1(\bU)}  \qD(\bV)$, $\lr = \mar$ and that we have {\bf conservative} updates. Recall from the preliminaries of the proof that given $\bU$, $\Up \in \Re^{m\times n}$ is such that $\Up = \marh \argmin \limits_{\bV \in \SP^1(\bU)}  \qD(\bV)$, meaning that $\min_{i\in[m],j\in[n]}|\Up_{ij}| \geq \marh$. Also recall that $
\pdim = \min\limits_{\bV \in \SP^1(\bU)}  \qD(\bV)$.

\begin{lemma}\label{ptl2}~\cite[Lemma 2.1]{tsuda2005matrix}
If $\bA\in\bS_{+}^d$ with eigenvalues in $[0,1]$ and $a\in\mathbb{R}$ then:
\[\left(1-e^a\right)\bs{A}\preceq \id-\exp(a\bs{A})\]
\end{lemma}

\begin{lemma}\label{malS}
For all trials $t\in \nmset$, we have:
\begin{equation}
\rent{\mat}{\wem{t}}-\rent{\mat}{\wem{t+1}}
\geq\lr y_t \tr{\mat\XT}+\left(1-e^{\lr y_t}\right)\tr{\wem{t}\XT}\,.
\end{equation}
\end{lemma}
\begin{proof}
We have:
\begin{align}
\rent{\mat}{\wem{t}}-\rent{\mat}{\wem{t+1}} &=\tr{\mat\log\wem{t+1}-\mat\log\wem{t}}+\tr{\wem{t}}-\tr{\wem{t+1}} \notag\\
\label{tseq11}& =\lr y_t \tr{\mat\XT}+\tr{\wem{t}}-\tr{e^{\log\wem{t}+\lr y_t \XT}}\\
\label{tseq21}&\geq\lr y_t \tr{\mat\XT}+\tr{\wem{t}}-\tr{e^{\log\wem{t}}e^{\lr y_t \XT}}\\
&=\lr y_t\tr{\mat\XT}+\tr{\wem{t}\left(\id-e^{\lr y_t \XT}\right)}\notag\\
\label{tseq31}&\geq\lr y_t \tr{\mat\XT}+\left(1-e^{\lr y_t}\right)\tr{\wem{t}\XT},
\end{align}
where Equation~\eqref{tseq11} comes from the update of the algorithm, Equation~\eqref{tseq21} comes from Lemma~\ref{ptl1} and Equation~\eqref{tseq31} comes from Lemma~\ref{ptl2} which applies since, by Lemma \ref{eigenx} all eigenvalues of $\XT$ are in $[0,1]$.
\end{proof}

\begin{lemma}\cite[Lemma A.5]{HPP16}\label{numl}
For $x\in[-1,1]$,
\[
x^2+x+1-e^{x}\geq(3-e)x^2\,.
\]
\end{lemma}

We proceed by showing that the ``progress'' $\rent{\mat}{\wem{t}}-\rent{\mat}{\wem{t+1}}$ of $\wem{t}$ towards $\mat$ may be further lower bounded by $c \mar$ (see Lemma~\ref{lem:detprogress}).  

\begin{lemma}\label{lem:detprogress}
Let $\co:=3-e$.   For all trials $t$ with $t \in \nmset$ (under the conditions of Lemma~\ref{lem:base}) we have:
\[
\rent{\mat}{\wem{t}}-\rent{\mat}{\wem{t+1}}\geq \co\mar^2
\]
\end{lemma}
\newcommand{\vvterm}{\mar}
\begin{proof}
By Lemma~\ref{lem:stuttS}, $\Up_{i_t j_t}=\tr{\mat{\XT}}-1$ so since $y_t = \sign(\bar{U}_{i_t j_t})$, and $\mar\leq |\bar{U}_{i_t j_t}|$ we have $\mar\leq y_t\left(\tr{\mat{\XT}}-1\right)$. So when $y_t=1$ we have $\tr{\mat{\XT}}\geq 1+\mar$ and when $y_t=-1$ we have $\tr{\mat{\XT}}\leq 1-\mar$. We use these inequalities as follows.

First suppose that $y_t=1$. By Lemma~\ref{malS} we have:
\begin{align}
\rent{\mat}{\wem{t}}-\rent{\mat}{\wem{t+1}}&\geq\vvterm\tr{\mat\XT}+\left(1-e^\mar\right)\tr{\wem{t}\XT}\notag\\
&\geq \vvterm\left(1+\vvterm\right)+\left(1-e^{\vvterm}\right)\tr{\wem{t}\XT} \notag \\
\label{stw1}&\geq \vvterm\left(1+\vvterm\right)+\left(1-e^\mar\right)\\
&=(\vvterm+\vvterm^2)+1-e^\mar\notag\\
\label{num1}&\geq\co\vvterm^2,
\end{align}
where Equation~\eqref{num1} comes from Lemma~\ref{numl} and Equation~\eqref{stw1} comes from the fact that $\hat{y}^t=-1$ and hence, by the algorithm, $\tr{\wem{t}\XT}\leq 1$.

Now suppose that $y_t=-1$. By Lemma \ref{malS} we have:
\begin{align}
\rent{\mat}{\wem{t}}-\rent{\mat}{\wem{t+1}}&\geq-\vvterm\tr{\mat\XT}+\left(1-e^{-\vvterm}\right)\tr{\wem{t}\XT}\notag\\
&\geq -\vvterm\left(1-\vvterm\right)+\left(1-e^{-\vvterm}\right)\tr{\wem{t}\XT} \notag\\
\label{stw2}&\geq -\vvterm\left(1-\vvterm\right)+\left(1-e^{-\vvterm}\right)\\
&=-\vvterm+\vvterm^2+1-e^{-\vvterm}\notag\\
\label{num2}&\geq\co\vvterm^2,
\end{align}
where Equation~\eqref{num2} comes from Lemma~\ref{numl} and Equation~\eqref{stw2} comes from the fact that $\hat{y}^t=1$ and hence, by the algorithm, $\tr{\wem{t}\XT}\geq 1$.
\end{proof}

\begin{lemma}\label{fpar1}
We have,
\[\co\mar^2\nm \le\rent{\mat}{\wem{1}}\,. 
\]
\end{lemma}
\begin{proof}
Suppose that we have $T$ trials. Then we have:
\begin{align}
\rent{\mat}{\wem{1}}&\geq\rent{\mat}{\wem{1}}-\rent{\mat}{\wem{T+1}}\notag\\
&=\sum_{t\in[T]}\left(\rent{\mat}{\wem{t}}-\rent{\mat}{\wem{t+1}}\right)\notag\\
&=\sum_{t\in\nmset}\left(\rent{\mat}{\wem{t}}-\rent{\mat}{\wem{t+1}}\right)\label{eq:aa}\\
&\geq\sum_{t\in\nmset}\co\mar^2\label{eq:bb}\\
&=\co\mar^2\nm \notag,
\end{align}
where~\eqref{eq:bb} follows from~\eqref{eq:aa} using Lemma~\ref{lem:detprogress}.
\end{proof}
\begin{lemma}\label{fpar2}
Given that $\wem{1}=\scp\frac{\id}{m+n}$ we have 
\[
\rent{\mat}{\wem{1}}\leq\tr{\mat}\log(m+n) + \trc(\mat)\log\frac{\trc(\mat)}{\scp} +\scp-\trc(\mat)
\]
\end{lemma}

\begin{proof}
We have:
\begin{align}
\rent{\mat}{\wem{1}}=&\tr{\mat\log\mat}-\tr{\mat\log\wem{1}}+\tr{\wem{1}}-\trc(\mat)\notag\\
=&\tr{\mat\log\mat}-\tr{\mat\log\left(\frac{\scp}{m+n}\id\right)}+\tr{\frac{\scp}{m+n}\id}-\trc(\mat)\notag\\
=&\tr{\mat\log\mat}-\tr{\mat\log\left(\frac{\scp}{m+n}\id\right)}+\scp-\trc(\mat)\notag\\
=&\tr{\mat\log\mat}-\tr{\mat\left(\id\log\left(\frac{\scp}{m+n}\right)\right)}+\scp-\trc(\mat)\notag\\
=&\tr{\mat\log\mat}-\tr{\mat\log\left(\frac{\scp}{m+n}\right)}+\scp-\trc(\mat)\label{eq:prelog}\\
\leq&\tr{\mat\log(\trc(\mat))}-\tr{\mat\log\left(\frac{\scp}{m+n}\right)}+\scp-\trc(\mat)\label{eq:postlog}\\
=&\left(\log(\trc(\mat))-\log\left(\frac{\scp}{m+n}\right)\right)\trc(\mat)+\scp-\trc(\mat)\notag\\
=&\log\left(\frac{\trc(\mat)(m+n)}{\scp}\right)\trc(\mat) +\scp-\trc(\mat)\,,\notag
\end{align}
where~\eqref{eq:postlog} follows from~\eqref{eq:prelog}, since $\mat := \bV \bLam \bV^{-1}$ where $\bLam$ is a diagonal matrix of the eigenvalues of $\mat$.  This holds since, 
\begin{align*}
\trc(\mat\log\mat) & = \trc(\bV \bLam \bV^{-1}\bV\log\bLam \bV^{-1}) \\  
&= \trc(\bV \bLam \log\bLam \bV^{-1})  \\
& = \trc(\bLam \log\bLam) \\
& = \sum_{i=1}^{m+n} \lambda_i \log(\lambda_i) \\
& \le (\sum_{i=1}^{m+n} \lambda_i) \log(\sum_{i=1}^{m+n} \lambda_i) \\
& = \trc(\mat\log(\trc(\mat)))\,.
\end{align*}

\end{proof}

\begin{lemma}\label{lem:base}
The mistakes, $\nm$, of Algorithm~\ref{alg:base}  with the assumption that $y_t = \sign(\bar{U}_{i_t j_t})$ for all $t\in\nmset$ and with parameters
$\mar\le\mc(\U)^{-1}$, $1\le \scp$ and $\lr = \mar$ and {\bf conservative} updates, is bounded above by:
\begin{equation}\label{eq:baselap}
\nm \le 3.6 \frac{1}{\mar^2}\left(\pdim\left(\log(m+n) + \log\frac{\pdim}{\scp}\right)+\scp-\pdim\right)
\end{equation}

\end{lemma}
\begin{proof}
Combining Lemmas~\ref{fpar1} and~\ref{fpar2} gives us
\[
\nm\leq\frac{1}{c}\frac{1}{\mar^2}\left(\tr{\mat}\log(m+n) + \trc(\mat)\log\frac{\trc(\mat)}{\scp} +\scp-\trc(\mat)
\right)
\]
Using Lemma \ref{lem:qexactS} and upper bounding $1/c$ by 3.6 then gives the result.
\end{proof}

The theorem statement for the realizable case then follows by setting $\scp \geq \pdim$. \hfill $\blacksquare$

\section{Proof of Theorem~\ref{thm:bndpdim}} 
We recall Theorem~\ref{thm:bndpdim} and then prove it.

\noindent
{\bf Theorem~\ref{thm:bndpdim}.}{\em \ \,
If  $\U\in\bikl$ define 
\begin{equation*}
\qDBF(\U) := 
\begin{cases}
 2 \trc( \bR^{\trans} \RN \bR)\RRN 
\!+ \! 2 \trc( \bC^{\trans} \CN \bC)\RCN\!+\! 2k\! +\! 2\ell \, & \RN \text{ and }\CN \text{ are PDLaplacians} \\
k\trace{\bR^{\trans} \bM \bR}\RRN+\ell\trace{\bC^{\trans} \bN \bC}\RCN
 & \RN\in\SPDM^m \text{ and }\CN\in\SPDM^n
\end{cases}.
\end{equation*}
as the minimum over all decompositions of $\U = \bR \Ustar \bC^{\trans}$ for
$\bR\in \BEM^{m,k}$, $\bC\in \BEM^{n,\ell}$ and $\Ustar \in \{-1,1\}^{k\times \ell}$.  
Thus for $\U\in\bikl$, 
\begin{align*}
\qD(\U) & \le \qDBF(\U)\quad\quad (\mbox{if }\maxnorm{\U} \le {1}/{\marh}) \\
\min_{\bV\in\SP^1(\U)} \qD(\bV) & \le \qDBF(\U) \quad\quad (\mbox{if }\mc(\U) \le {1}/{\marh})\,.
\end{align*}
}
\begin{proof}
A {\em $\marh$-decomposition} of matrix $\U$ is given by a $\nP \in \RNM^{m,d}$ and a $\nQ \in \RNM^{n,d}$ such that $\nP \nQ^{\trans} = \marh \U$.  
A {\em block-invariant decomposition} of matrix $\U\in\bikl$  is given by a $\nP \in \RNM^{m,d}$ and a $\nQ \in \RNM^{n,d}$ for some $d$ such that there exists a
$\delta\in (0,1]$, 
$\nPstar\in \RNM^{k,d}$, and a $\nQstar\in \RNM^{\ell,d}$,  so that 
$\nP = \R \nPstar$, $\nQ = \C \nQstar$ and $\nP \nQ^{\trans} = \delta \U$.  

We now prove the following intermediate result,
\begin{quote}
{\bf Lemma:} If $\U\in\bikl$, then for every $\marh\in (0,1/\maxnorm{\U})$, there exists a block-invariant $\marh$-decomposition of $\U$.
\begin{proof}
Since $\U\in\bikl$ we have that $\U = \bR \Ustar \bC^{\trans}$ for some
$\bR\in \BEM^{m,k}$, $\bC\in \BEM^{n,\ell}$ and $\Ustar \in \{-1,1\}^{k\times \ell}$.  Observe by block invariance we have that $\maxnorm{\U} = \maxnorm{\Ustar}$ and by the definition of $\maxnorm{\cdot}$ we have that there exists a $\left(\frac{1}{\maxnorm{\U}}\right)$-decomposition of $\Ustar$ via factors $\nPstar\in \RNM^{k,d}$, and a $\nQstar\in \RNM^{\ell,d}$, this implies that $\nP := \R \nPstar$, $\nQ := \C \nQstar$ is a
$\left(\frac{1}{\maxnorm{\U}}\right)$-block-invariant decomposition of $\U$.  

Now given any $\marh\in (0,1/\maxnorm{\U})$ we construct a $\marh$-block-invariant decomposition of $\U$.
Set $c:= \marh \maxnorm{\U}$. We construct new factor matrices $\nP'\in \RNM^{m,d+1}$ and $\nQ'\in \RNM^{n,d+1}$
\[\nP' :=
\begin{pmatrix}
c \nP & (\sqrt{1-c^2}) \bone 
\end{pmatrix};\quad\quad \nQ' :=
\begin{pmatrix}
\nQ & \bzero
\end{pmatrix}\,.
\]
Observe that $(\nP',\nQ')$ is the required $\marh$-block-invariant decomposition of $\U$.
\end{proof}
\end{quote}

Recall~\eqref{eq:defpdim},
\begin{equation}\label{eq:reca}
\qD(\bU) := \min_{\nP \nQ^\trans = {\marh}\U} \RRN\trc\left(\nP^{\trans}\RN\nP  \right) 
+
\RCN\trc\left(\nQ^{\trans}\CN\nQ \right)\,.
\end{equation}
Observe that when the feasible set of the optimization that defines $\qD(\bU)$ is non-empty and $\U\in\bikl$,  there exists a member of the feasible set which is a block-invariant decomposition of $\U$ by the lemma above. We proceed by proving an upper bound of 
\[
\RRN\trc\left(\nP^{\trans}\RN\nP  \right) 
+
\RCN\trc\left(\nQ^{\trans}\CN\nQ \right)
\]
for every block-invariant decomposition of $\U$.

First we will bound the term $\trc(\nP^{\trans}\RN\nP)$ for general positive definite matrices and  then for PDLaplacians. By symmetry, the bound will also hold for $\trc(\nQ^{\trans}\CN\nQ)$.

Suppose $(\nP,\nQ)$ is a block-invariant decomposition of $\U$. Then, we have
\begin{align*}
\trc(\nP^\trans \bM \nP) & = \trc( (\R\nPstar)^{\trans} \bM \R\nPstar ) = \trc(\nPstar(\nPstar)^\trans\R^\trans\bM \R ) \\ & \le  \trc(\nPstar(\nPstar)^\trans)\trc(\R^\trans\bM \R) = k \trc(\R^\trans\bM \R)\,,
\end{align*}
where the inequality comes from the fact that $\trc(\bA\bB) \le\lambda_{\text{max}}(\bA)\trc(\bB)\le\trc(\bA)\trc(\bB)$ for $\bA,\bB\in \SPSDM$.  By symmetry we have demonstrated the inequality for positive definite matrices.

We now consider PDLaplacians.  Assume $\RN := \bLb = \la+\left(\frac{\one}{\m}\right)\left(\frac{\one}{\m}\right)^\trans\RAD_{\bL}^{-1}$, a PDLaplacian.  Recall the following two elementary inequalities from the preliminaries: if
$\bu\in [-1,1]^m$, then 
\begin{align}
(\bu^{\trans} \bL \bu) \RAD_{\bL} & \le \onehalf (\bu^{\trans} \bLb \bu)  \RAD_{\bL^{\circ}}\,, \label{eq:aamj}
\\
(\bu^{\trans} \bLb \bu)  \RAD_{\bL^{\circ}} & \le 2 (\bu^{\trans} \bL \bu\, \RAD_{\bL} +  1)\,. \label{eq:bbmj}
\end{align}
Observe that for an $m \times m$ graph Laplacian $\bL$ with adjacency matrix $\bA$ that for $\bX\in\Re^{m\times d}$,
\begin{equation}\label{eq:lapquad}
\trc(\bX^{\trans} \bL \bX) = \sum_{(i,j) \in E} A_{ij} \norm{\bX_i - \bX_j}^2\,.
\end{equation}
Suppose $(\nP,\nQ)$ is a block-invariant decomposition of $\U$ then
the row vectors $\nP_1,\ldots,\nP_m$ come in at most $k$ distinct varieties, that is $|\bigcup_{i\in [m]} \nP_i| \le k$.  The same holds for $\bR$ and furthermore $(\nP_r = \nP_s) \Longleftrightarrow (\bR_r = \bR_s)$ for $r,s\in [m]$.  Observe that given $r,s\in [m]$ that if $\bR_r \ne \bR_s$ then $\norm{\bR_r -\bR_s}^2 =2$ and $\norm{\nP_r -\nP_s}^2 \le 4$ since they are coordinate and unit vectors respectively.  This then implies,
\begin{equation}\label{eq:beworld}
\trc\left(\nP^{\trans}\bL\nP  \right) \le 2\trc\left(\bR^{\trans}\bL\bR  \right)\,. 
\end{equation}
Thus we have
\begin{align*}
\trc(\nP^{\trans}\RN\nP) \RRN &\le 2 \trc(\nP^{\trans}\bL\nP)\RAD_{\bL} +2k \quad &\, \text{by}~\eqref{eq:bbmj}\,  \\
&\le 4 \trc\left(\bR^{\trans}\bL\bR  \right)\RAD_{\bL} +2k \quad &\, \text{by}~\eqref{eq:beworld}\,  \\
&\le 2 \trc\left(\bR^{\trans}\RN\bR  \right)\RRN+2k \quad &\, \text{by}~\eqref{eq:aamj}\, 
\end{align*}
By symmetry we have demonstrated the inequality for PDLaplacians. 
\end{proof}

\section{Proofs for Section \ref{sec:inductive}}\label{sec:append_inductive}

\subsection{Proof of Proposition \ref{eq:alg}}
\newcommand{\zed}[1]{\bs{v}(#1)}
\newcommand{\zet}[1]{\bs{\bar{v}}(#1)}
\newcommand{\XTf}{\tilde{\bs{X}}^{T+1}}
\newcommand{\ai}[1]{a_{#1}}
\newcommand{\len}{l}
\newcommand{\les}{q}
\newcommand{\expo}{q}
\newcommand{\cost}[1]{\sigma_{#1}}
\newcommand{\zbt}{\bar{z}_t}
\newcommand{\zbs}{\bar{z}_s}
\newcommand{\ybs}{\bar{y}_s}
\newcommand{\funk}[1]{f_s(#1)}

We now prove Proposition \ref{eq:alg}, that the transductive and inductive algorithms are equivalent.  Recall by assumption that  $\IRRN=\RRN$ and $\IRCN = \RCN$.

\subsubsection{Equivalence of Traces}

Suppose, in this subsection, that we have some given trial $t$. In this subsection we analyse the inductive algorithm. We make the following definitions:

\begin{definition}
For all $s\in\uset\cap[t]$
$$\zed{s}:= \con{\frac{(\sqrt{(\RN^t)^+})\be^{i_s}}{\sqrt{2\RRN}}}{\frac{(\sqrt{(\CN^t)^+})\be^{j_s}}{\sqrt{2\RCN}}}$$
$$\zet{s}:= \con{\frac{(\sqrt{(\RN^{T+1})^+})\be^{i_s}}{\sqrt{2\RRN}}}{\frac{(\sqrt{(\CN^{T+1})^+})\be^{j_s}}{\sqrt{2\RCN}}}$$
\end{definition}

Note that $\XT(s)=\zed{s}\zed{s}^{\trans}$ and $\XTf(s)=\zet{s}\zet{s}^{\trans}$ for $s\in\uset\cap[t]$.

\begin{lemma}\label{prodlem}
For all $\len\in\mathbb{N}$ and for all $\ai{1},\ai{2}, ..., \ai{l}\in\uset\cap[t-1]$ there exists some $\alpha\in\mathbb{R}$ such that:
$$\XT(\ai{1})\XT(\ai{2})\cdots\XT(\ai{\len})=\alpha\zed{\ai{1}}\zed{\ai{\len}}^{\trans}$$
and
$$\XTf(\ai{1})\XTf(\ai{2})\cdots\XTf(\ai{\len})=\alpha\zet{\ai{1}}\zet{\ai{\len}}^{\trans}$$
\end{lemma}

\begin{proof}
We prove by induction on $\len$. In the case $\len:=1$ the result is clear with $\alpha:=1$.

Now suppose the result holds with $\len:=\les$ for some $q\in\mathbb{N}$. We now show that it holds for $\len:= \les+1$. Since it holds for $\len:=\les$, choose $\alpha'$ such that $\XT(\ai{1})\XT(\ai{2})\cdots\XT(\ai{\les})=\alpha'\zed{\ai{1}}\zed{\ai{\les}}^\trans$ and $\XTf(\ai{1})\XTf(\ai{2})\cdots\XTf(\ai{\les})=\alpha'\zet{\ai{1}}\zet{\ai{\les}}^\trans$. Note that we now have:
\begin{align*}
\XT(\ai{1})\XT(\ai{2})\cdots\XT(\ai{\len})&=\XT(\ai{1})\XT(\ai{2})\cdots\XT(\ai{\les})\XT(\ai{\len})\\
&=\alpha'\zed{\ai{1}}\zed{\ai{\les}}^{\trans}\XT(\ai{\len})\\
&=\alpha'\zed{\ai{1}}\zed{\ai{\les}}^{\trans}\zed{\ai{\len}}\zed{\ai{\len}}^\trans\\
&=\left(\zed{\ai{\les}}^{\trans}\zed{\ai{\len}}\right)\alpha'\zed{\ai{1}}\zed{\ai{\len}}^\trans\\
&=\left(\frac{\RNfunc(i_{\ai{\les}},i_{\ai{\len}})}{2\RRN}+\frac{\CNfunc(j_{\ai{\les}},j_{\ai{\len}})}{2\RCN}\right)\alpha'\zed{\ai{1}}\zed{\ai{\len}}^{\trans}\\
\end{align*}
Similarly we have: $$\XTf(\ai{1})\XTf(\ai{2})\cdots\XTf(\ai{\len})=\left(\frac{\RNfunc(i_{\ai{\les}},i_{\ai{\len}})}{2\RRN}+\frac{\CNfunc(j_{\ai{\les}},j_{\ai{\len}})}{2\RCN}\right)\alpha'\zet{\ai{1}}\zet{\ai{\len}}^{\trans},\\$$
from which the result follows.
\end{proof}

\begin{lemma}\label{prodlem3}
For all $\len\in\mathbb{N}$ and for all $\ai{1},\ai{2}, ..., \ai{l}\in\uset\cap[t-1]$ we have:
$$\tr{\XT(t)\XT(\ai{1})\XT(\ai{2})\cdots\XT(\ai{\len})}=\tr{\XTf(t)\XTf(\ai{1})\XTf(\ai{2})\cdots\XTf(\ai{\len})}$$
\end{lemma}

\begin{proof}
By Lemma \ref{prodlem}, let $\alpha$ be such that \[\XT(\ai{1})\XT(\ai{2})\cdots\XT(\ai{\len})=\alpha\zed{\ai{1}}\zed{\ai{\len}}^{\trans}\] and \[\XTf(\ai{1})\XTf(\ai{2})\cdots\XTf(\ai{\len})=\alpha\zet{\ai{1}}\zet{\ai{\len}}^{\trans}.\] Note that:
\begin{align*}
&\tr{\XT(t)\XT(\ai{1})\XT(\ai{2})\cdots\XT(\ai{\len})}\\
&=\alpha\tr{\XT(t)\zed{\ai{1}}\zed{\ai{\len}}^{\trans}}\\
&=\alpha\tr{\zed{t}\zed{t}^{\trans}\zed{\ai{1}}\zed{\ai{\len}}^{\trans}}\\
&=\alpha\tr{\zed{\ai{\len}}^{\trans}\zed{t}\zed{t}^{\trans}\zed{\ai{1}}}\\
&=\alpha\left(\zed{\ai{\len}}^{\trans}\zed{t}\right)\left(\zed{t}^{\trans}\zed{\ai{1}}\right)\\
&=\alpha\left(\frac{\RNfunc(i_{\ai{\len}},i_{t})}{2\RRN}+\frac{\CNfunc(j_{\ai{\len}},j_{t})}{2\RCN}\right)\left(\frac{\RNfunc(i_{t},i_{\ai{1}})}{2\RRN}+\frac{\CNfunc(j_{t},j_{\ai{1}})}{2\RCN}\right)
\end{align*}
Similarly we have: 
\begin{align}
&\tr{\XTf(t)\XTf(\ai{1})\XTf(\ai{2})\cdots\XTf(\ai{\len})}\\
=&\alpha\left(\frac{\RNfunc(i_{\ai{\len}},i_{t})}{2\RRN}+\frac{\CNfunc(j_{\ai{\len}},j_{t})}{2\RCN}\right)\left(\frac{\RNfunc(i_{t},i_{\ai{1}})}{2\RRN}+\frac{\CNfunc(j_{t},j_{\ai{1}})}{2\RCN}\right)
\end{align} The result follows.
\end{proof}

\begin{lemma}\label{prodlem2}
For any $\expo\in\mathbb{N}$, any $\kappa\in\mathbb{R}^+$ and any $b_1, b_2,\cdots b_{t-1}\in\mathbb{R}$ we have:
$$\tr{\XT(t)\left(\sum_{s\in\uset\cap[t-1]}b_s\XT(s)\right)^{\expo}}=\tr{\XTf(t)\left(\sum_{s\in\uset\cap[t-1]}b_s\XTf(s)\right)^{\expo}}$$
\end{lemma}

\begin{proof}
We have:
\begin{align*}
&\tr{\XT(t)\left(\sum_{s\in\uset\cap[t-1]}b_s\XT(s)\right)^{\expo}}\\
&=\tr{\XT(t)\sum_{\ai{1}\in\uset\cap[t-1]}\sum_{\ai{2}\in\uset\cap[t-1]}\cdots\sum_{\ai{\expo}\in\uset\cap[t-1]}\left(\prod_{i=1}^{\expo}b_{\ai{i}}\right)\XT(\ai{1})\XT(\ai{2})\cdots\XT(\ai{\expo})}\\
&=\sum_{\ai{1}\in\uset\cap[t-1]}\sum_{\ai{2}\in\uset\cap[t-1]}\cdots\sum_{\ai{\expo}\in\uset\cap[t-1]}\left(\prod_{i=1}^{\expo}b_{\ai{i}}\right)\tr{\XT(t)\XT(\ai{1})\XT(\ai{2})\cdots\XT(\ai{\expo})}\\
\end{align*}
and similarly,
\begin{align*}
&\tr{\XTf(t)\left(\sum_{s=1}^{t-1}b_s\XTf(s)\right)^{\expo}}\\
&=\tr{\XTf(t)\sum_{\ai{1}\in\uset\cap[t-1]}\sum_{\ai{2}\in\uset\cap[t-1]}\cdots\sum_{\ai{\expo}\in\uset\cap[t-1]}\left(\prod_{i=1}^{\expo}b_{\ai{i}}\right)\XTf(\ai{1})\XTf(\ai{2})\cdots\XTf(\ai{\expo})}\\
&=\sum_{\ai{1}\in\uset\cap[t-1]}\sum_{\ai{2}\in\uset\cap[t-1]}\cdots\sum_{\ai{\expo}\in\uset\cap[t-1]}\left(\prod_{i=1}^{\expo}b_{\ai{i}}\right)\tr{\XTf(t)\XTf(\ai{1})\XTf(\ai{2})\cdots\XTf(\ai{\expo})}.\\
\end{align*}
The result follows by Lemma \ref{prodlem3}.
\end{proof}

\begin{lemma}\label{prodlemfin}
For any $\kappa\in\mathbb{R}^+$ and any $b_1, b_2,\cdots b_{t-1}\in\mathbb{R}$ we have:
$$\tr{\XT(t)\exp\left(\kappa\bs{I}+\sum_{s\in\uset\cap[t-1]}b_s\XT(s)\right)}=\tr{\XTf(t)\exp\left(\kappa\bs{I}+\sum_{s\in\uset\cap[t-1]}b_s\XTf(s)\right)}$$
\end{lemma}

\begin{proof}
Using the fact that $\exp\left(\bA + \bB\right) = \exp\left(\bA\right)\exp\left(\bB\right)$ for commuting matrices $\bA$ and $\bB$, and noting that the multiple of the identity matrix commutes with any matrix, we have that
\[ \tr{\XT(t)\exp\left(\kappa\bs{I}+\sum_{s\in\uset\cap[t-1]}b_s\XT(s)\right)} = \tr{\XT(t)\exp\left(\kappa\bs{I}\right) \exp\left(\sum_{s\in\uset\cap[t-1]}b_s\XT(s)\right)}.\]
By the Taylors series expansion we have:
\begin{align*}
\tr{\XT(t)\exp\left(\kappa\bs{I}\right) \exp\left(\sum_{s\in\uset\cap[t-1]}b_s\XT(s)\right)}&=e^{\kappa} \tr{\XT(t)\sum_{\expo=0}^{\infty}\frac{1}{\expo !}\left(\sum_{s\in\uset\cap[t-1]}b_s\XT(s)\right)^{\expo}}\\
&=e^{\kappa} \sum_{\expo=0}^{\infty}\frac{1}{\expo !}\tr{\XT(t)\left(\sum_{s\in\uset\cap[t-1]}b_s\XT(s)\right)^{\expo}}
\end{align*}
Similarly, we have \[\tr{\XTf(t)\exp\left(\kappa\bs{I}+\sum_{s\in\uset\cap[t-1]}b_s\XTf(s)\right)}=e^{\kappa} \sum_{\expo=0}^{\infty}\frac{1}{\expo !}\tr{\XTf(t)\left(\sum_{s\in\uset\cap[t-1]}b_s\XTf(s)\right)^{\expo}}.\]
The result then follows from Lemma \ref{prodlem2}.
\end{proof}

\subsubsection{Equivalence of Algorithms}

On a trial $t$ let $\zbt$ be the prediction ($\ybt$) of the inductive algorithm and let $\ybt$ remain the prediction of the transductive algorithm. We fix $\kappa:=\log\left(\scp/(m+ n)\right)$.

\begin{lemma}\label{expred}
On a trial $t$ the prediction, $\ybt$, of the transductive algorithm is given by:
$$\ybt=\tr{\XTf(t)\exp\left(\kappa\bs{I}+\sum_{s=1}^{t-1}\funk{\ybs}\XTf(s)\right)}$$
and the prediction, $\zbt$, of the inductive algorithm is given by:
$$\zbt=\tr{\XT(t)\exp\left(\kappa\bs{I}+\sum_{s=1}^{t-1}\funk{\zbs}\XT(s)\right)}$$
where $\funk{x}:=\lr{y_s}$ if $y_sx \leq [\mbox{\sc non-conservative}] \times \marh $ and $\funk{x}:=0$ otherwise.
\end{lemma}

\begin{proof}
Direct from algorithms, noting that if $s\notin\uset\cap[t-1]$ then $\funk{\zbs}=0$.
\end{proof}

\begin{lemma}\label{yeqz}
Given a trial $t$, if $\ybs=\zbs$ for all $s<t$, then $\ybt=\zbt$.
\end{lemma}

\begin{proof}
Direct from Lemmas \ref{expred} and \ref{prodlemfin} (with $b_s:=\funk{\ybt}=\funk{\zbt}$), noting that if $s\notin\uset\cap[t-1]$ then $\funk{\zbs}=0$.
\end{proof}

Proposition \ref{eq:alg} follows by induction over Lemma \ref{yeqz}.\hfill $\blacksquare$

\subsection{Proof of Proposition~\ref{lem:min_kernel_norm}}
In the following, we define $\cK_{\bxx}(\cdot):=\cK(\bxx,\cdot)$. 
 If $r\geq2$, $\delta^*:=\min \left( 2,  \frac{1}{4}\delta(S_1,\ldots, S_k) \right)$. This implies that $\delta^*\leq\min \left( 2,  \frac{r-1}{2r}\delta(S_1,\ldots, S_k) \right)$.
Recall that  $s(\bxx):=\frac{r-1}{2r} \bxx+ \frac{r+1}{2}\bone $.
Then observe that, given that the transformation $\bxt_i =s(\bxx_i)$ holds true for all $i \in [m]$, requiring $S_1, \ldots, S_k \subset [-r,r]^d$ with $\bxx_1,\ldots,\bxx_m\in\cup_{i=1}^k S_i$ and $\delta^* \leq \min \left(2,  \frac{r-1}{2r}\delta(S_1,\ldots, S_k) \right)$ is equivalent to the requirement that $\St_1, \ldots, \St_k \subset [1,r]^d$ with $\bxt_1,\ldots,\bxt_m\in\cup_{i=1}^k \St_i$ and $\delta^* \leq \min \left(2, \delta(\St_1,\ldots, \St_k) \right)$. Furthermore, for all $i\in[m]$ and $j\in[k]$, we have that $\bxx_i\in S_j$ if and only if $\bxt_i\in\St_j$. We shall proceed with the latter set of requirements for simplicity.
Recall that the RKHS for the $d=1$ min kernel $\mathcal{H}^1_{\cK}$  is the set of all absolutely continuous functions from $[0,\infty)^d \rightarrow \mathbb{R}$ that satisfy $f(0)= 0 $ and $\int_0^\infty [f'(x)]^2 \mathrm{d}x < \infty$. 

\begin{lemma}
The inner product for $f \in \mathcal{H}^1_{\cK}$ may be computed by,
\[ \langle f, g \rangle = \int_0^\infty f'(x) g'(x) \mathrm{d}x\,.\]
\end{lemma}
\begin{proof}
We show this by the reproducing property:
\[ \langle f, \cK_{x} \rangle= f(x). \]
Defining $\bm{1}_x(t)$ as the step function that evaluates to 1 for $t\leq x$ and $0$ otherwise, we note that the derivative of $\min(x,t)$ with respect to $t$ is equal to $\bm{1}_{x}(t)$. This gives rise to
\[ \int_0^\infty f'(t) \cK'(x,t) \mathrm{d}t = \int_0^\infty f'(t) \bm{1}_{x}(t)   \mathrm{d}t= \int_0^{x} f'(t)  \mathrm{d}t = f(x)\,. \]
Using the condition of $f(0) = 0$, we then obtain the reproducing property.
\end{proof}

\begin{figure}
\centering
\includegraphics[width=1.2\figurewidth, height=1.2\figureheight]{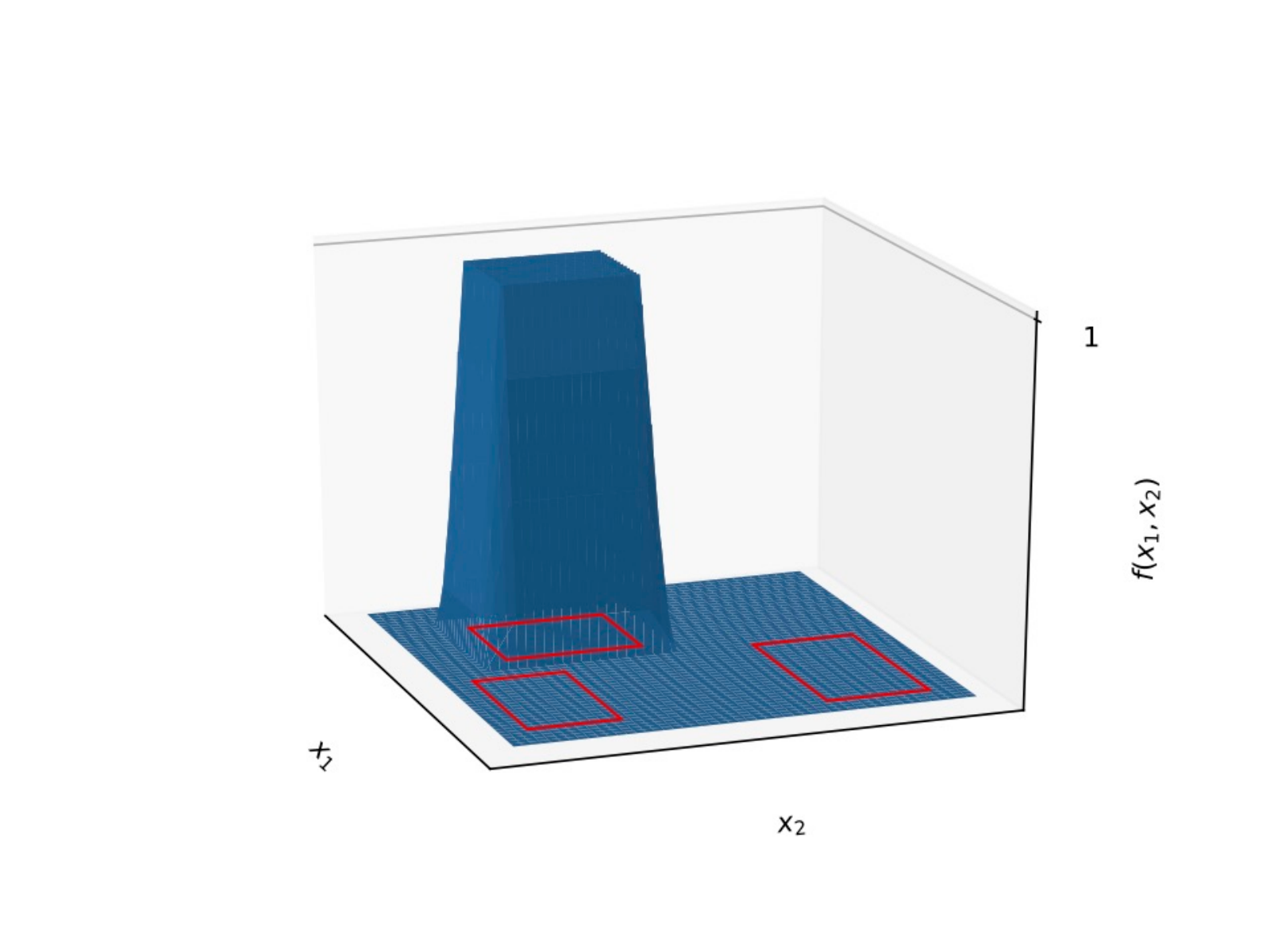}
\caption{Visualization of the function $f(x_1,x_2)$ with $S_1$, $S_2$ and $S_3$ represented as red rectangles in the $x_1-x_2$ plane.}
\label{fig:box}
\end{figure}

\begin{lemma}\label{lem:norm_f_bound}
Given $k$ boxes $\St_1,\ldots,\St_k \subset [1,r]^d$, $\delta^* \leq \min \left(2,  \delta(\St_1,\ldots, \St_k) \right)$ and  $\bxt_1,\ldots,\bxt_m\in\cup_{i=1}^k \St_i$, there exists a function $f \in H_{\cK}$ for which $f(\bxt_j) = [\bxt_j \in \St_1]$ for $j\in[m]$ and this function has norm
\[ ||f||^2= \left(\frac{4}{\delta^*}\right)^d. \]
\end{lemma}
\begin{proof}
Recall that a {\em box} in $\Re^d$ is a set $\{\bxx : a_i \le x_i \le b_i, i\in [d]\}$ defined by a pair of vectors $\ba,\bb\in\Re^d$. 
First, we consider the case of $d=1$, with the coordinates of $\St_1$ defined by $a$ and $b$. Defining the function that interpolates the points $\bxt_1,\ldots \bxt_m$ in one dimension as $f^1 \in \mathcal{H}^1_{\cK}$, we chose $f^1$ to be the following:
\[f^1(x) = 
\begin{cases}
 0 & \mathrm{for }\quad x \leq a-\frac{\delta^*}{2} \\
 \frac{2}{\delta^*} x +1- \frac{2}{\delta^*} a  &  \mathrm{for }\quad a-\frac{\delta^*}{2} <x \leq a \\
1 & \mathrm{for } \quad  a < x \leq b \\
   -\frac{2}{\delta^*} x +1 + \frac{2}{\delta^*} b &  \mathrm{for }\quad b <x \leq b+\frac{\delta^*}{2} \\
0 &    \mathrm{for }  \quad x >  b+\frac{\delta^*}{2}.
\end{cases}
\]This function is picked from the space $\mathcal{H}_{\cK}^1$ so that $\int_0^\infty[(f^1)'(x)]^2 \mathrm{d}x$ is minimized with respect to ``worst-case'' constraints. The condition on $\delta^*$ implies that $\delta^*\leq2$, so that $f^1(0) = 0$. It also implies that $\delta^* \leq \delta(S_1,\ldots S_k)$ so that for all $i \in [m]$, $f^1(\tilde{x}_i) = 0$ if $\tilde{x}_i  \notin S_1$.
The norm $||f^1||^2$, then becomes
\begin{align*}
||f^1||^2 & = \int_0^\infty |(f^1)'(x)|^2 \mathrm{d}x \\
&= \int_{a-\frac{\delta^*}{2}}^a \left(\frac{2}{\delta^*}\right)^2 \mathrm{d}x + \int_{b}^{b+\frac{\delta^*}{2}} \left(\frac{2}{\delta^*}\right)^2 \mathrm{d}x \\
&= 2 \left(\frac{2}{\delta^*}\right)^2 \left(\frac{\delta^*}{2}\right)= \frac{4}{\delta^*}\,.
\end{align*}

This can be extended to multiple dimensions by observing that the induced product norm of $f$ is the product of the norms of $f^1$ in each dimension, thus giving the required bound. In this case also, the condition on $\delta^*$ ensures both $f(\bzero)=0$ and $f(\bxt_i) = 0$ for $\bxt_i  \notin \St_1$,where $i\in[m]$.  For an illustration of this function in two dimensions, see Figure~\ref{fig:box}.
\end{proof}
\begin{lemma}\label{lem:min_kernel_normsingle}
Given $k$ boxes $\St_1,\ldots,\St_k \subset [1,r]^d$, $\delta^* \leq \min \left(2,  \delta(\St_1,\ldots, \St_k) \right) $ and  $\bxt_1,\ldots,\bxt_m\in\cup_{i=1}^k \St_i$,  if $\bu = (u_i = [\bxt_i\in \St_1])_{i\in [m]}$ and  $\bK= (\cK(\bxt_i,\bxt_j))_{i,j\in [m]}$ 
then $\bu^{\trans} \bK^{-1} \bu  \le \left(\frac{4}{\delta^*}\right)^d$.
\end{lemma}
\begin{proof}
Using Lemma~\ref{lem:norm_f_bound} we observe that,
\[
\bu^{\trans} \bK^{-1}\bu = \argmin_{f\in H_{\cK} : f(\bxt_i) = [\bxt_i\in \St_1], i\in [m]} \norm{f}_{\cK}^2 \le \argmin_{f\in H_{\cK} : f(\bxt) = [\bxt\in \St_1], \bxt\in \cup_{i\in [k]} \St_i} \norm{f}_{\cK}^2\le\left(\frac{4}{\delta^*}\right)^d\,,
\] for $\bu := (u_i = [\bxt_i\in \St_1])_{i\in [m]}$, $\bK := (\cK(\bxt_i,\bxt_j))_{i,j\in [m]}$ and $\bxt_1,\ldots,\bxt_m \in \cup_{i\in [k]} \St_i$. Note that the second term in the equation above has a constraint in terms of the given $m$ points $\bxt_1, \ldots, \bxt_m$, whereas the optimization in the third term has a similar constraint, but in terms of any $\bxt$ that satisfies $ \bxt\in \cup_{i\in [k]} \St_i$. \end{proof}
Defining $\bu_i = (\bR^\trans_i)^\trans$, then observe that the term $\trace{\bR^{\trans} \bK^{-1} \bR}= \sum_{i\in [k]} \bu_i^\trans \bK^{-1} \bu_i$.
Thus by applying Lemma~\ref{lem:min_kernel_normsingle} to each $\bu_i$,  we have that $ \trace{\bR^{\trans} \bK^{-1} \bR} \le k(\frac{4}{\delta^*})^d$.  \hfill $\blacksquare$
\end{document}